\documentclass[reprint,aip]{revtex4-2}

\usepackage{physics,amsmath,soul,color, amssymb, cancel, amsthm, bbm}
\usepackage{wasysym, enumitem, graphicx, wrapfig, multirow, makecell, float, hyperref, bm}
\usepackage{algorithm, algpseudocode}

\DeclareMathOperator*{\argmin}{arg\,min}

\newtheorem{theorem}{Theorem}

\newcommand{\Cov}{\mathrm{Cov}}

\newcommand{\CalI}{\mathcal{I}}
\newcommand{\CalJ}{\mathcal{J}}
\newcommand{\expectation}[1]{\left< #1 \right>}
\newcommand{\bigo}[1]{\mathcal{O}\left( #1 \right)}
\newcommand{\bfS}{\mathbf{\Sigma}}
\newcommand{\bfth}{\pmb{\theta}}
\newcommand{\bff}{\bm{f}}
\newcommand{\bfg}{\bm{g}}
\newcommand{\bfw}{\mathbf{w}}
\newcommand{\bfx}{\mathbf{x}}
\newcommand{\bfX}{\mathbf{X}}
\newcommand{\bfy}{\mathbf{y}}
\newcommand{\bfp}{\mathbf{p}}
\newcommand{\bfP}{\mathbf{P}}
\newcommand{\bfq}{\mathbf{q}}

\algnewcommand{\Initialize}[1]{%
  \State \textbf{Initialize:} 
  \Statex \hspace*{\algorithmicindent}\parbox[t]{0.88\linewidth}{\raggedright #1}
}

\begin{document}

\title{An information-matching approach to optimal experimental design and active learning}
\author{Yonatan Kurniawan}
\affiliation{Brigham Young University, Provo, UT 84602, USA}
\author{Tracianne B.~Neilsen}
\affiliation{Brigham Young University, Provo, UT 84602, USA}
\author{Benjamin L.~Francis}
\affiliation{Achilles Heel Technologies, Orem, UT 84097, USA}
\author{Alex M.~Stankovic}
\affiliation{SLAC National Accelerator Laboratory, Menlo Park, CA, USA}
\author{Mingjian Wen}
\affiliation{University of Electronic Science and Technology of China, Chengdu, 611731, China}
\author{Ilia Nikiforov}
\author{Ellad B.~Tadmor}
\affiliation{University of Minnesota, Minneapolis, MN 55455, USA}
\author{Vasily V.~Bulatov}
\author{Vincenzo Lordi}
\affiliation{Lawrence Livermore National Laboratory}
\author{Mark K.~Transtrum}
\email{mktranstrum@byu.edu}
\affiliation{Brigham Young University, Provo, UT 84602, USA}
\affiliation{Achilles Heel Technologies, Orem, UT 84097, USA}
\affiliation{SLAC National Accelerator Laboratory, Menlo Park, CA, USA}

\begin{abstract}
    The efficacy of mathematical models heavily depends on the quality of the training data, yet collecting sufficient data is often expensive and challenging.
    Many modeling applications require inferring parameters only as a means to predict other quantities of interest (QoI).
    Because models often contain many unidentifiable (sloppy) parameters, QoIs often depend on a relatively small number of parameter combinations.
    Therefore, we introduce an information-matching criterion based on the Fisher Information Matrix to select the most informative training data from a candidate pool.
    This method ensures that the selected data contain sufficient information to learn only those parameters that are needed to constrain downstream QoIs.
    It is formulated as a convex optimization problem, making it scalable to large models and datasets.
    We demonstrate the effectiveness of this approach across various modeling problems in diverse scientific fields, including power systems and underwater acoustics.
    Finally, we use information-matching as a query function within an Active Learning loop for material science applications.
    In all these applications, we find that a relatively small set of optimal training data can provide the necessary information for achieving precise predictions.
    These results are encouraging for diverse future applications, particularly active learning in large machine learning models.
\end{abstract}

\maketitle

A model's predictive performance depends strongly on the quality and quantity of data available for training.
Curating comprehensive datasets, however, often confronts practical constraints, including instrumentation, available resources, and cost.
Optimal experimental design (OED) \cite{leardi_experimental_2009} and active learning (AL) \cite{alizadeh_survey_2021} emerge as practical data collection strategies.
Intentionally designing maximally informative experiments guarantees that data are most informative relative to the underlying phenomena of interest, minimize costs, and meet operational requirements.
These methodologies have broad applications across scientific domains, including sensor placement problems in power systems \cite{yuill_optimal_2011,peng_optimal_2006,aminifar_contingency-constrained_2010_edited,milosevic_nondominated_2003,chakrabarti_optimal_2008_edited,soudi_optimal_1999}
and underwater acoustics \cite{wood_optimisation_2003,dosso_optimal_1999,barlee_array_2002,barlee_array_2002,dosso_array_2006,tidwell_designing_2019},
the development of accurate interatomic potentials in materials science \cite{csanyi_learn_2004_edited,artrith_high-dimensional_2012_edited,podryabinkin_active_2017,gubaev_accelerating_2019},
and many other scientific fields \cite{casey2007optimal,Jeong_Zhuang_Transtrum_Zhou_Qiu_2018,wang_active-learning_2021,kim_uncertainty_2022_edited}.

Many OED criteria utilize the Fisher information matrix (FIM), whose inverse establishes a lower bound on parameter covariance, known as the Cram\'{e}r-Rao bound \cite{cramer1999mathematical,streiner_precision_2006,van2007parameter}.
Common approaches optimize parameter precision through the FIM, for example by minimizing its trace (A-optimality) \cite{jacroux_-optimality_1989,butler_approximate_2008,jones_-optimal_2021}, maximizing its determinant (D-optimality) \cite{andere-rendon_design_1997,balsa-canto_computing_2008,podryabinkin_active_2017,jones_-optimal_2021}, or maximizing the smallest eigenvalues (E-optimality) \cite{chow_e-optimality_1997,balsa-canto_computing_2008,morgan_e-optimality_2011}.

However, many applications of predictive models do not require precise parameter estimates \emph{per se}, but accurate predictions for key quantities of interest (QoIs) \cite{casey2007optimal,transtrum2012optimal,wen_force-matching_2017_edited,shmueli_predictive_2011,tavazza_uncertainty_2021}.
This distinction is well illustrated by \emph{sloppy models}, where many parameter combinations are \emph{practically unidentifiable}, yet still yield precise predictions \cite{transtrum_geometry_2011,transtrum2015perspective,brouwer_underlying_2018,quinn_information_2022}.
Such models typically have only a few stiff, identifiable directions, while the rest are sloppy and poorly constrained.
Furthermore, the parameters that are identifiable from the training data may not align with those most relevant for predicting QoIs.
Precise predictions are still possible when these QoI-relevant directions fall within the identifiable subspace; however, if they do not, the resulting uncertainty may be large or even divergent, regardless of how well some parameters are constrained \cite{kurniawan_bayesian_2022}.

Furthermore, sloppy models exhibit a characteristic information spectrum, where FIM eigenvalues are nearly uniformly spaced on a log scale over many orders of magnitude.
Many eigenvalues are smaller than the model's evaluation precision, rendering the aforementioned OED criteria sensitive to numerical noise.


\begin{figure*}[ht]
    \centering
    \includegraphics[width=0.9\textwidth]{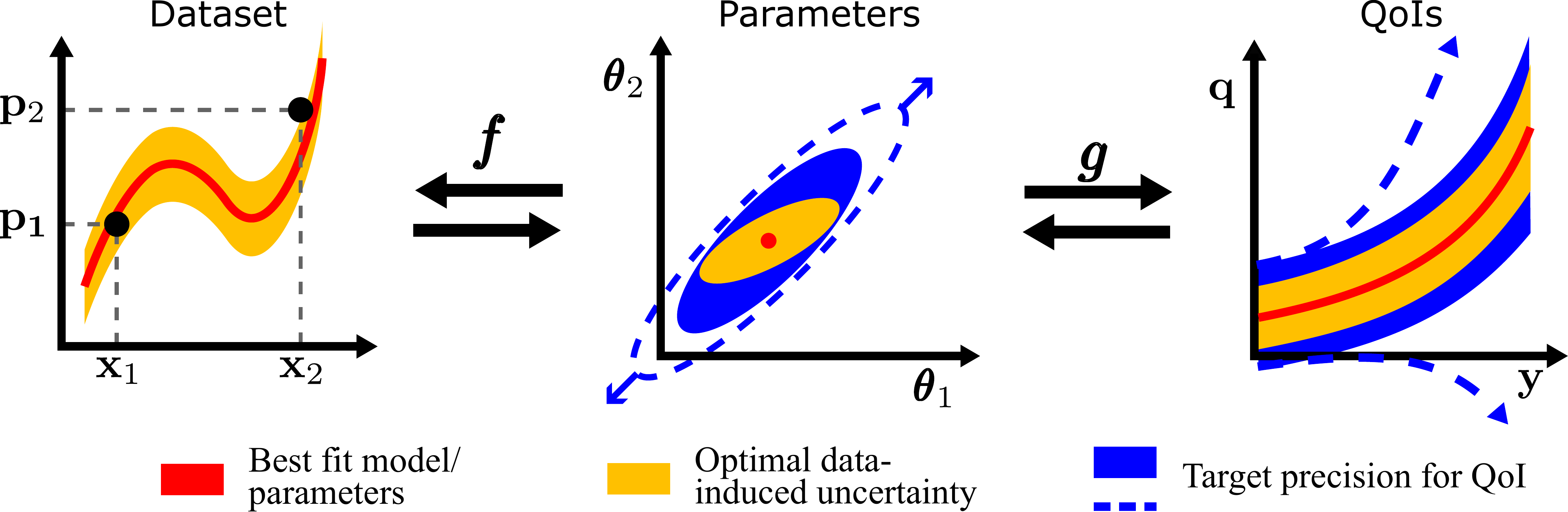}
    \caption[The information-matching workflow]{
	\textbf{Relationship between training data (left), model parameters (middle), and QoIs (right) in the information-matching framework.}
	One first selects the target precision for the QoIs (blue envelope in the right panel).
	This QoI precision induces a minimal confidence region in parameter space (blue ellipse in the middle panel).
	The information-matching criterion selects training data and target precision (orange envelope in the left panel) such that the resulting parameter uncertainty (orange ellipse in the middle panel) is more restrictive than that induced by the QoIs.
	Propagating the parameter uncertainty to the QoIs gives predictions that are at least as precise as the original target (orange envelope in the right panel).
	This relationship holds even if the target uncertainties were divergent for certain QoIs (dashed blue curves in the right panel), resulting in the target parameter confidence diverging for some parameter combinations (dashed blue ellipse in the middle panel, extending in some directions).
    }
    \label{fig:modeling}
\end{figure*}

Motivated by these considerations, we develop an information-matching method for OED that prioritizes the precision of predictions for the QoIs.
This approach uses a model parameterized by $\bfth$ in two key scenarios: training, $\bff$, and prediction, $\bfg$ (see Fig.~\ref{fig:modeling}).
Given a dataset of $M$ independent inputs $\{\bfx_m\}_{m=1}^M$ and their corresponding ground truth labels $\{\bfp_m\}_{m=1}^M$, we first use $\bff(\bfth; \bfx_m)$ to train the parameters against the data $\{\bfx_m, \bfp_m\}$.
Then, we use the trained parameters to predict the QoIs $\bfq$ corresponding to the input $\bfy$ through the mapping $\bfg(\bfth; \bfy)$.
In this scenario, $\bfy$ may act as an input, similar to $\bfx_m$, as a control parameter for the QoIs, or as a discrete index to distinguish between different QoI values.
The information-matching method leverages both scenarios to identify the minimal subset of training data that contains the information necessary to precisely constrain the parameters relevant to QoIs.
Our strategy is: given a target precision for QoIs, align the FIM for the training data with that of the QoIs, ensuring the training data carry the information needed to constrain the predictions precisely.
Thus, only the parameter combinations that need to be identified are trained, bypassing numerical stability issues in cases where the FIMs are ill-conditioned.

The FIM is defined as the expectation value of the Hessian of the log-likelihood over the probability of the labels.
For weighted least-squares, by far the most common regression scenario, the negative log-likelihood is (up to an additive constant)
\begin{equation}
    \label{eq:cost_function}
    \ell(\bfth) = \frac{1}{2} \sum_{m=1}^M w_m \lVert \bfp_m - \bff(\bfth; \bfx_m) \rVert_2^2,
\end{equation}
where the weight $w_m$ is the inverse variance of the label $\bfp_m$, i.e., $w_m = 1/\sigma_m^2$.
As shown in the supplementary material, the FIM for the training scenario is given by
\begin{equation}
    \label{eq:fim_configs}
    \CalI (\bfth) = \sum_{m=1}^M w_m \CalI_m(\bfth),
    \quad
    \CalI_m(\bfth) = J_{\bff}^T(\bfth; \bfx_m) J_{\bff}(\bfth; \bfx_m),
\end{equation}
where $\CalI_m$ and $J_{\bff}(\bfth; \bfx_m)$ are the FIM and the Jacobian matrix of $\bff(\bfth; \bfx_m)$ with respect to the parameters $\bfth$ corresponding to the $m$-th datum, respectively \cite{transtrum_geometry_2011}.
This equation highlights a generic, fundamental property of the FIM that the expected information in the entire training data is the sum of information from each independent datum.

Notably, the FIM denotes the expected information over the probability of the labels and does not depend on the observed label value $\bfp_m$.
This property means the FIM can be evaluated for any model predictions, including downstream applications for which ground truth labels are not available, and accounts for its broad appeal in OED.
The FIM then quantifies how much information about the model parameters is required to achieve a target precision.
We denote the target precision of the QoIs $\bfq$ by the covariance matrix $\bfS$; the FIM for the QoIs is given as
\begin{equation}
    \label{eq:fim_qoi}
    \CalJ(\bfth) = J_{\bfg}^T(\bfth) \bfS^{-1} J_{\bfg}(\bfth),
\end{equation}
where $J_{\bfg}(\bfth)$ is the Jacobian matrix of the proxy $\bfg(\bfth; \bfy)$ with respect to the parameters $\bfth$.

The information-matching method leverages the FIMs of both the training data and QoIs (Eqs.~(\ref{eq:fim_configs}) and (\ref{eq:fim_qoi}), respectively) to guarantee that the information in the training data is sufficient to achieve the target precision in the QoIs.
To select the minimal set of training data that achieves this minimal information bound, we solve the following convex problem for the weight vector $\bfw = \begin{bmatrix} w_1 & w_2 & \dots & w_M \end{bmatrix}^T$:
\begin{equation}
    \label{eq:convex_opt}
    \begin{aligned}
	& \text{minimize} && \Vert \bfw \Vert_1 \\
	& \text{subject to} && w_m \geq 0, \\
	& && \CalI = \sum_{m=1}^M w_m \CalI_m \succeq \CalJ.
    \end{aligned}
\end{equation}
We conjecture that for many practical problems, the key information required for precise predictions is contained in a few key data points; therefore, we design the objective function to minimize the $\ell_1$-norm of the weight vector to encourage sparse solutions.
The non-zero weights identify the most important data and the precision with which the labels must be measured to ensure the target precision in the QoIs.

The matrix inequality constraint in Eq.~(\ref{eq:convex_opt}) is crucial for ensuring the target precision on the QoIs.
Formally, it means that the difference $\CalI - \CalJ$ is positive semidefinite.
Intuitively, it indicates that fitting the down-selected data results in smaller parameter variance compared to fitting the QoIs directly, as illustrated in Fig.~\ref{fig:modeling}.
Theorem~\ref{thm:information_matching} formalizes this statement (proof in the SM).

\begin{theorem}
    \label{thm:information_matching}
    Let $\bfg(\bfth; \bfy)$ denote a mapping from the model parameters $\bfth$ to the QoIs for input $\bfy$ that is analytic at $\bfth_0 = \expectation{\bfth}_{\bfth}$, where $\expectation{\cdot}_{\bfth}$ denotes an expectation value over the distribution of parameters.
    Consider parameters of the form $\bfth = \bfth_0 + \epsilon \delta \bfth$.
    If the constraints in Eq.~\eqref{eq:convex_opt} are satisfied, then
    \begin{equation}
	\label{eq:preds_constraint}
	\Cov(\bfg) \preceq \bfS + \bigo{\epsilon^3},
    \end{equation}
    where $\bfS$ is the target covariance of the QoIs.
\end{theorem}

Theorem~\ref{thm:information_matching} states that the uncertainties of the QoIs propagated from the optimal training data [$\Cov(\bfg)$] are within the predefined target uncertainties ($\bfS$), up to third order in $\epsilon$.
The information-matching method is unique in that it simultaneously aims to minimize data usage while ensuring adequate information for precise predictions.
In contrast, the previously mentioned OED criteria (A-, D-, and E-optimality) only prioritize reducing some measure of parameter variance.
By minimizing the number of data points, the information-matching approach not only enhances efficiency but also improves model interpretability by focusing the analysis on only the most critical training data.

We first demonstrate the information-matching method for optimally placing sensors (Phasor Measurement Unit, or PMU) in a power system network.
The goal is to use a few strategically placed PMUs to infer the complex-valued voltages (system states) at every bus.
A PMU placed on a bus measures the bus voltage and currents in adjoining branches, synchronized with GPS time stamps.
By measuring voltages and currents, it is possible to achieve full-state observability without requiring PMUs at every bus.

\begin{figure}[!hbt]
    \centering
    \includegraphics[width=0.45\textwidth]{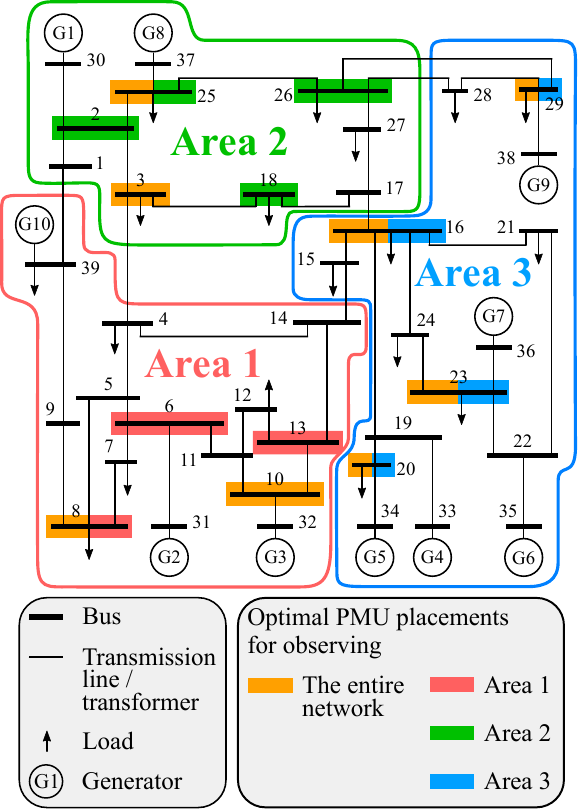}
    \caption[The IEEE 39-bus system]{			 
	\textbf{The IEEE 39-bus system.}
	Buses are represented by thick black lines, while transmission lines and transformers are shown as thin lines connecting them.
        Loads are depicted as black arrows pointing outward from the buses, and the circled labels G1 through G10 indicate the generators.
        Buses highlighted in orange denote the optimal PMU placements for full observability of the entire network.
        Buses highlighted in other colors (red, green, and blue) represent the optimal PMU placements for partial observability in the corresponding area.
	Many buses are double-highlighted with orange and another color, showing overlaps between full and partial observability.
	Non-overlapping optimal buses result from unobserved branches.
    }
    \label{fig:pws_39bus}
\end{figure}

We use the IEEE 39-bus system \cite{ieee39bus}, represented graphically in Fig.~\ref{fig:pws_39bus}, as a benchmark.
Nodes denote buses, and edges represent transmission lines and transformers \cite{yuill_optimal_2011}.
The parameters $\bfth$ to be inferred are the system states (voltage magnitude and angle), and $\bff$ represents PMU measurements.
Full parameter identifiability implies that the QoIs are the state variables themselves [$\bfg(\bfth; \bfy) = \bfth$] and the QoI FIM is non-singular.
We enforce this condition by setting $\CalJ = \lambda I$ for some small $\lambda > 0$.
The information-matching condition implies that $\CalI$ is also non-singular, and all states are observable.
The optimal PMU placements for this problem are represented as the orange highlighted buses in Fig.~\ref{fig:pws_39bus}.
On this initial benchmark test, the information-matching method naturally selects the same buses identified in previous studies \cite{milosevic_nondominated_2003,chakrabarti_optimal_2008_edited}, even without preassigning PMU locations.
Furthermore, this result is broadly consistent with prior findings \cite{peng_optimal_2006,aminifar_contingency-constrained_2010_edited,yuill_optimal_2011}.

In practice, the power system analysts model reduced portions of the full network, focusing on regions under their direct control.
However, these areas are influenced by states outside the target areas, and identifying an appropriate reduced area equivalent is a challenging task \cite{saric2018data,zhao_power_2019_edited}.
We next partition the IEEE 39-bus system into non-overlapping regions (indicated by red, green, and blue in  Fig.~\ref{fig:pws_39bus}) and seek a minimal set of sensors to achieve observability within each area \cite{matpowerDescriptionCase39},
 without regard for states outside it.
We implement this by setting the diagonal elements of $\CalJ$ corresponding to external states to zero, allowing infinite uncertainties for those states.
The optimal PMU locations for each area are shown in Fig.~\ref{fig:pws_39bus} as buses are highlighted in different colors based on their respective areas.
Notably, there are overlaps (double-highlighted buses) between optimal PMU placements for full and (the union of) three subnetworks, while non-overlapping locations are a consequence of enforcing observability for each of the subnetworks separately.

\begin{figure}[!hbt]
    \centering
    \includegraphics[width=0.45\textwidth]{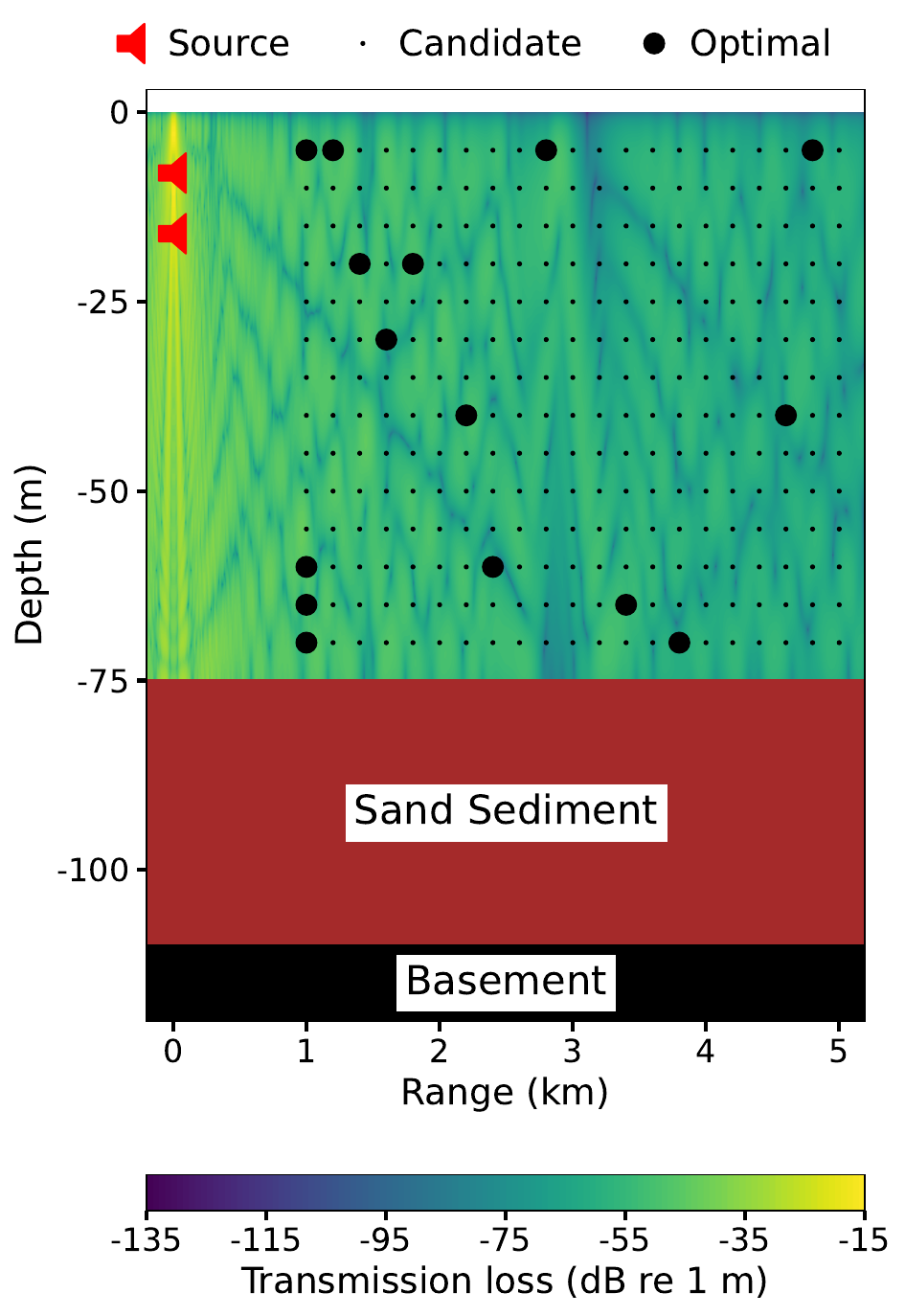}
    \caption[Source localization in a shallow ocean]{
	\textbf{Source localization in a shallow ocean.}
        Optimal receiver locations for localizing two sound sources (red speakers) in a shallow ocean with a sandy seabed using transmission loss data at 200~Hz.
        Small dots indicate candidate sites; large dots are the optimal receiver locations.
    }
    \label{fig:orca_sand_f200}
\end{figure}

Next, we consider an optimal sensor placement problem in passive acoustic source localization in the ocean.
The objective is to determine the optimal sensor (sound receiver) locations to infer the location of sound sources in a shallow ocean.
This problem is difficult because sound propagation depends in complicated ways on unknown properties of the ocean environment, including the water temperature and the sediments in the seabed.
Traditionally, source localization requires first estimating the ocean environmental parameters, as is done in matched-field processing \cite{tolstoy:1993:matched}.
In contrast, information-matching allows us to learn only those combinations of acoustic parameters that are necessary to infer the source location, avoiding full environmental inversion when it is not required.

We aim to localize two vertically separated sound sources at depths of 8 and 16~m (red speakers) within $\pm$ 2.5~m vertically and $\pm$ 100~m horizontally (Fig.~\ref{fig:orca_sand_f200}).
Candidate receivers are arranged in a rectangular grid (small dots), motivated by common practices of using vertical and horizontal line arrays for ocean sound measurements \cite{wood_optimisation_2003,dosso_optimal_1999,barlee_array_2002,dosso_array_2006,buck_information_2002}.
We use a range-independent normal mode model, called ORCA \cite{westwood_normal_1996}, to simulate the sound propagation in the ocean and compute the transmission loss at 200~Hz for each candidate location illustrated as small dots in Fig.~\ref{fig:orca_sand_f200} (details in the SM).
The model parameters $\bfth$ for simulating sound propagation include the source and receiver locations, as well as the parameterization of the ocean environment.
The ocean environment is modeled with 75~m-deep water above an ocean floor consisting of a sandy sediment layer on a half-space basement layer.

To localize the two sources using Eq.~(\ref{eq:convex_opt}), the QoI FIM $\CalJ$ is a diagonal matrix, where diagonal elements that correspond to source locations are set to their inverse target precision, and all other elements are zero, indicating that only the relevant environmental parameters for localizing the sources are constrained as needed.
The optimal receiver locations are shown as the large dots in Fig.~\ref{fig:orca_sand_f200}, which account for only 5\% of the candidate locations.

The FIM is a local quantity that can vary (sometimes significantly) for different parameter values.
In the preceding examples, we assumed a reasonable prior estimate of the models' parameters (e.g., the bus voltage phasor and the ocean environmental parameters), eliminating the need to recalculate the FIM after fitting with the optimal data.
However, parameters often exhibit significant variability in response to new data, so we need to optimize the parameters alongside the data.
To address this, we extend the OED problem to an Active Learning (AL) strategy and use Eq.~(\ref{eq:convex_opt}) as a data query function.

\begin{algorithm}[H]
    \caption{Active learning via information-matching}
    \label{alg:active_learning}
    \begin{algorithmic}[1]
	\Initialize{
	    $\bfX \gets \{\bfx_m\}_{m=1}^M$  \Comment{Candidate input data} \\
	    $\bfP \gets \text{empty}(M)$  \Comment{To store labels} \\
	    $\bfw^{\text{opt}} \gets \text{zeros}(M)$  \Comment{To store optimal weights} \\
	    $\bfth_0 \gets \bfth$  \Comment{Initial parameters}
	}
        \While{True}
	    \State Compute $\CalJ(\bfth_0)$ using Eq.~(\ref{eq:fim_qoi})
	    \For{$m = 1:M$}
		\State Compute $\CalI_m (\bfth_0)$ using Eq.~(\ref{eq:fim_configs})
	    \EndFor

	    \State $\bfw \gets $ Solve Eq.~(\ref{eq:convex_opt})
	    \State $\bfw^{\text{opt}} \gets \{ \max(w_m^{\text{opt}}, w_m), ~\forall m = 1:M \}$
	    \If{$\bfw^{\text{opt}}$ converge}
		\State break
	    \Else
		\ForAll{$ \{m\,\vert\, w_m^{\text{opt}} > 0\} $}
		\State $\bfp_m \gets$ Generate label for $\bfx_m$
		
		\EndFor
		\State  $\bfth_0 \gets \argmin_{\bfth} \ell(\bfth)$  \Comment{Update parameters}
	    \EndIf

	\EndWhile
    \end{algorithmic}
\end{algorithm}

Algorithm~\ref{alg:active_learning} outlines the iterative AL process based on information-matching.
We begin by preparing a pool of candidate inputs $\bfX = \{\bfx_m\}_{m=1}^M$ and initializing the parameters to the \emph{a~priori} best estimate $\bfth_0$.
The procedure starts by evaluating $\CalJ$ and $\CalI_m$ at $\bfth_0$ for all $m$ inputs and solving Eq.(\ref{eq:convex_opt}) for the weights $\bfw$.
For each datum, we update the optimal weight by comparing the new optimal value with the current one, retaining the larger one.
This step ensures that the amount of information in the training data in subsequent iterations is non-decreasing.
Convergence occurs when the change in optimal weights between subsequent iterations is below some chosen threshold.
If not converged, we generate labels for data with nonzero weights and update the parameters by minimizing Eq.~(\ref{eq:cost_function}), then iterate.

As an example, we apply this AL algorithm to the development of interatomic potentials in materials science.
Interatomic potentials are crucial in atomistic simulations as they approximate the interaction energy between atoms \cite{Tadmor_Modeling_Materials}.
These potentials are typically trained on small-scale quantities, e.g., energy and atomic forces obtained from computationally expensive first-principles calculations \cite{ercolessi_interatomic_1994,wen_force-matching_2017_edited}, then used in larger-scale simulations to predict material properties.
Despite the scale discrepancy between training and prediction, the dynamics of atoms primarily depend on their local neighborhoods.
Thus, our objective is to identify training data (atomic configurations) that are maximally informative about the atomic neighborhoods for precise material predictions.

We apply Algorithm~\ref{alg:active_learning} to develop an optimal 15-parameter Stillinger--Weber (SW) potential for molybdenum disulfide (MoS$_2$) to precisely predict the energy ($E$) as a function of lattice parameter ($a$) under uniform in-plane strain \cite{wen_force-matching_2017_edited,kurniawan_bayesian_2022,OpenKIM_SW_MoS2_driver,OpenKIM_SW_MoS2}.
The predictions are shifted by $E_c$ (the energy at the equilibrium lattice constant $a_0$) to align the minimum with the origin, effectively showing strain-induced energy changes.
The candidate dataset comprises 2000 atomic configuration snapshots from an \emph{ab-initio} molecular dynamics trajectory at 750~K, each with 96 Mo and 192 S atoms \cite{wen_force-matching_2017_edited, wen_dataset_2024}.
While this dataset contains force labels, in practice, the candidate dataset does not need to include labels; Algorithm~\ref{alg:active_learning} generates them on demand for the optimal configurations.
We choose the target precision to be 10\% of the values predicted by the potential trained on the full dataset \cite{wen_force-matching_2017_edited, wen_dataset_2024}.
In Fig.~\ref{fig:ips_swmos2}, we compare the uncertainties of the QoI obtained from the optimal configurations (red envelope) with the target uncertainty (blue envelope).
Our findings indicate that seven atomic configurations suffice to constrain the parameters and achieve the target precision.

While some details of the calculation---the selected configurations, optimal weights, and optimal parameters---depend on the choice of initial parameters, these differences do not affect the validity of the final estimated uncertainty.
As long as the problem is feasible, the resulting prediction uncertainty estimated by the FIM is guaranteed to be within the target uncertainty.
This dependence on initialization is examined in more detail in the SM, where we provide numerical examples based on the SW MoS$_2$ model.
Additional results for other SW potentials for silicon are also provided in the SM.

\begin{figure}[!hbt]
    \centering
    \includegraphics[width=0.45\textwidth]{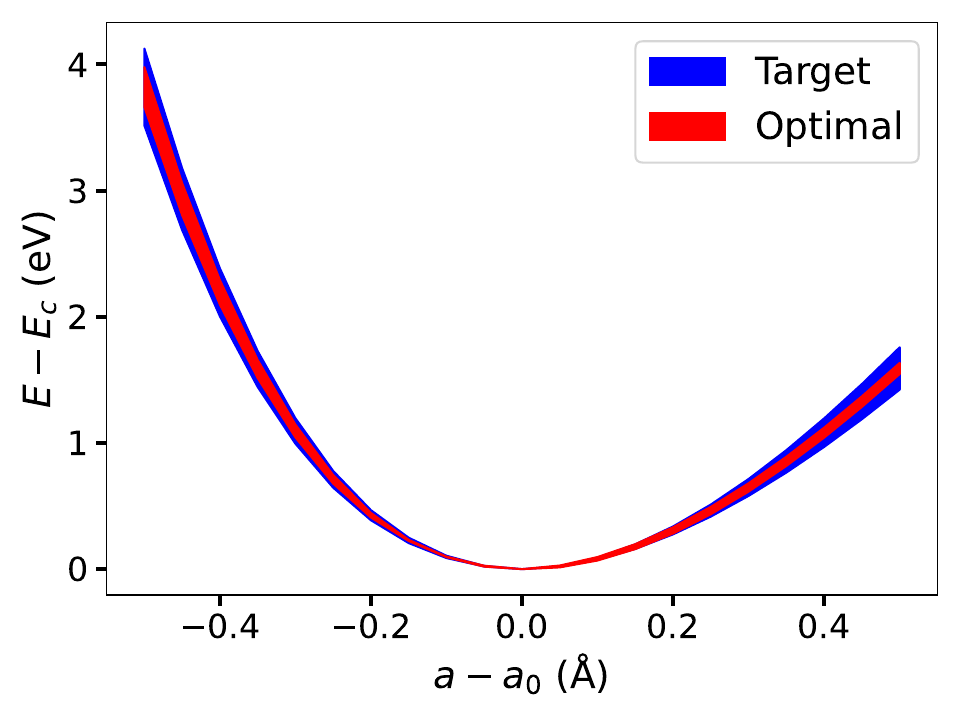}
    \caption[Uncertainty in the energy ($E$) of monolayer MoS$_2$ versus in-plane lattice parameter ($a$)]{
	\textbf{Uncertainty in the energy ($E$) of monolayer MoS$_2$ versus in-plane lattice parameter ($a$)}.
	Predictions are shifted by the energy $E_c$ at the equilibrium lattice constant $a_0$, aligning the minimum with the origin.
	The blue envelope is the target uncertainty (10\% of the values predicted by the potential trained on the full dataset). In contrast, the red envelope shows the uncertainty propagated from the seven optimal training atomic configurations.
	Notice that the optimal propagated uncertainty is smaller than the target uncertainty.
    }
    \label{fig:ips_swmos2}
\end{figure}



In summary, we have introduced an information-matching method to identify a minimal set of informative data points to meet target precision requirements for downstream QoIs.
We have demonstrated the versatility of the approach across diverse domains, including power system networks, underwater acoustics, and AL for interatomic potential development.
Unlike classical FIM-based OED criteria that aim at reducing global parameter uncertainty, information-matching optimizes a fundamentally different objective: it matches the information content of the data to the precision required for the QoIs.
As a result, only those parameters that influence the QoI are constrained---and only to the extent needed to satisfy the target QoI precision---while others may remain unconstrained if they are irrelevant.
This distinction highlights that comparisons with traditional OED methods must be interpreted carefully, as the two frameworks prioritize different aspects of the inference problem and therefore tend to yield different designs.
The accompanying theorem additionally formalizes this perspective and shows that QoI precision can be controlled without explicitly estimating the full parameter uncertainty.
This opens new possibilities in settings where full parameter identifiability is unnecessary or infeasible, such as PMU placement for partially observable networks or localization of ocean acoustic sources without environmental inversion.
Broader applications are likely in areas such as biology, neuroscience, geology, and atmospheric science, where models often contain many weakly identifiable parameters but well-defined QoIs \cite{quinn_information_2022}.

Despite these advantages, successfully solving Eq.~(\ref{eq:convex_opt}) depends on the richness of the initial candidate dataset; if any parameter that is relevant for the QoIs cannot be constrained by any of the candidate datasets, the optimization is infeasible.
Because Eq.~(\ref{eq:convex_opt}) is convex, algorithms will detect when a problem is infeasible.
In such a case, it would be desirable to augment the candidate datasets with more informative candidates.  
Identifying such additional data points to achieve feasibility remains an open, problem-dependent question.
Our current practical approach is to generate large and diverse candidate datasets.
This can often be done quickly and cheaply.
For example, in interatomic potential development, many atomic configurations with varied lattice parameters can be generated, while in sensor placement problems, one may define a fine grid of potential sensor locations.
Because Eq.~(\ref{eq:convex_opt}) is convex and scales favorably with problem size, this approach is practical in many scenarios.

Future work may extend the method to larger models and machine-learning applications, including machine-learned interatomic potentials.
Investigating the feasibility and potential advantages of such integration could unlock broader applications and insights.
Finally, a deeper theoretical analysis of the limiting behavior of the optimization may offer valuable guidance on robustness and further broaden applicability across diverse scientific domains.

\section*{Supplementary Material}

The supplementary material provides the mathematical formulation of the least-squares problem and the associated uncertainty quantification framework.
It also includes a proof of the information-matching theorem, detailed descriptions of the models used in the example test cases, and additional results for those models.

\begin{acknowledgments}
    YK, MKT, IN, EBT, and AMS acknowledge partial support through NSF under Grants No.\ 1834251, 1834332, and 2223986.
    YK, IN, EBT, VVB, VL, and MKT acknowledge funding support from the Laboratory Directed Research and Development program (project code 23-SI-006) and a special computational time allocation on the Lassen supercomputer from the Computational Grand Challenge program at Lawrence Livermore National Laboratory.
    Portions of this work were performed under the auspices of the U.S.\ Department of Energy by Lawrence Livermore National Laboratory under Contract DE-AC52-07NA27344.
    TBN and MKT acknowledge that this work is related to the Department of the Navy Award N00014-24-12566.
    Any opinions, findings, conclusions, or recommendations expressed in this material are those of the authors and do not necessarily reflect the views of the Office of Naval Research.
    Some of the calculations were conducted using computational facilities provided by the Brigham Young University Office of Research Computing.
\end{acknowledgments}

\section*{Data Availability Statement}

The data that support the findings of this study are openly available in Figshare at \href{http://doi.org/10.6084/m9.figshare.27274029.v1}{http://doi.org/10.6084/m9.figshare.27274029.v1}, reference number~[\citenum{Kurniawan2024}].
The corresponding code is available in Github, reference number~[\citenum{information-matching_repo}].

\section*{AUTHOR DECLARATION}
\subsection*{Conflict of Interest}
The authors have no conflicts to disclose.
\subsection*{Author Contributions}
\textbf{Yonatan Kurniawan:} Conceptualization (equal); Formal Analysis (equal); Investigation (lead); Methodology (equal); Software (lead); Writing---original draft (lead).
\textbf{Tracianne Neilsen:} Data curation (equal); Writing---original draft (equal).
\textbf{Benjamin Francis:} Data curation (equal); Writing---original draft (equal).
\textbf{Alex Stankovic:} Writing---review \& editing (equal).
\textbf{Mingjian Wen:} Data curation (equal); Writing---review \& editing (equal).
\textbf{Ilia Nikiforov:} Writing---review \& editing (supporting).
\textbf{Ellad Tadmor:} Writing---review \& editing (equal).
\textbf{Vasily Bulatov:} Writing---review \& editing (equal).
\textbf{Vincenzo Lordi:} Writing---review \& editing (equal).
\textbf{Mark Transtrum:} Conceptualization (equal); Formal analysis (equal); Methodology (equal); Writing---original draft (equal).

\bibliographystyle{unsrt}
\bibliography{refs.bib,refs_zotero.bib}

@Misc{elliott:tadmor:2011,
  author       = {Ryan S. Elliott and Ellad B. Tadmor},
  title        = {{K}nowledgebase of {I}nteratomic {M}odels ({KIM}) Application Programming Interface ({API})},
  howpublished = {\url{https://openkim.org/kim-api}},
  publisher    = {OpenKIM},
  year         = 2011,
  doi          = {10.25950/ff8f563a},
}

@Misc{OpenKIM_SW_MoS2,
  author       = {Mingjian Wen},
  title        = {{M}odified {S}tillinger-{W}eber potential ({MX}2) for monolayer {M}o{S}2 developed by {W}en et al. (2017) v001},
  doi          = {10.25950/eeedbbc4},
  howpublished = {OpenKIM, \url{https://doi.org/10.25950/eeedbbc4}},
  publisher    = {OpenKIM},
  year         = 2018,
}

@Misc{OpenKIM_SW_MoS2_driver,
  author       = {Mingjian Wen},
  title        = {{S}tillinger-{W}eber {M}odel {D}river for {M}onolayer {MX}2 systems v001},
  doi          = {10.25950/7d664757},
  howpublished = {OpenKIM, \url{https://doi.org/10.25950/eeedbbc4}},
  publisher    = {OpenKIM},
  year         = 2018,
}

@Book{Tadmor_Modeling_Materials,
  author = {E. B. Tadmor and R. E. Miller},
  publisher = {Cambridge University Press},
  title = {Modeling Materials: {C}ontinuum, Atomistic and Multiscale Techniques},
  year = {2011},
}

@article{Tadmor_Elliott_Sethna_Miller_Becker_2011,
    title={The potential of atomistic simulations and the {K}nowledgebase of {I}nteratomic {M}odels},
    volume={63},
    ISSN={1543-1851},
    DOI={10.1007/s11837-011-0102-6},
    number={7},
    journal={JOM},
    author={Tadmor, E. B. and Elliott, R. S. and Sethna, J. P. and Miller, R. E. and Becker, C. A.},
    year={2011},
    month={Jul},
    pages={17–17}
}

@article{Jeong_Zhuang_Transtrum_Zhou_Qiu_2018,
    title={Experimental design and model reduction in systems biology}, 
    volume={6},
    ISSN={2095-4697},
    DOI={10.1007/s40484-018-0150-9},
    number={4},
    journal={Quantitative Biology},
    author={Jeong, Jenny E. and Zhuang, Qinwei and Transtrum, Mark K. and Zhou, Enlu and Qiu, Peng},
    year={2018},
    month={Dec},
    pages={287–306}
}

@article{Wen_Afshar_Elliott_Tadmor_2022,
    title={KLIFF: A framework to develop physics-based and machine learning interatomic potentials}, 
    volume={272},
    ISSN={0010-4655},
    DOI={10.1016/j.cpc.2021.108218},
    journal={Computer Physics Communications},
    author={Wen, Mingjian and Afshar, Yaser and Elliott, Ryan S. and Tadmor, Ellad B.},
    year={2022},
    month={Mar},
    pages={108218}
}

@book{horn_johnson_2013,
    author={Horn, Roger A. and Johnson, Charles R.},
    place={Cambridge},
    title={Positive definite matrices},
    isbn={9780521839402; 0521839408; 9780521548236; 0521548233},
    booktitle={Matrix Analysis},
    edition={2nd ed.},
    publisher={Cambridge University Press, New York},
    year={2013},
}

@book{bhatia_2007,
    Author = {Rajendra, Bhatia},
    ISBN = {9780691129181},
    Publisher = {Princeton University Press},
    Series = {Princeton Series in Applied Mathematics},
    Title = {Positive Definite Matrices.},
    URL = {https://lib.byu.edu/remoteauth/?url=https://search-ebscohost-com.byu.idm.oclc.org/login.aspx?direct=true&AuthType=ip&db=e025xna&AN=273046&site=ehost-live&scope=site},
    Year = {2007},
    }

@Misc{SW_OpenKIM_MO,
  author       = {Amit K. Singh and Frank H. Stillinger and Thomas A. Weber},
  title        = {{S}tillinger-{W}eber potential for {S}i due to {S}tillinger and {W}eber (1985) v006},
  doi          = {10.25950/dd263fe3},
  howpublished = {OpenKIM, \url{https://doi.org/10.25950/dd263fe3}},
  publisher    = {OpenKIM},
  year         = 2021,
}

@Misc{SW_OpenKIM_MD,
  author       = {Mingjian Wen and Yaser Afshar and Frank H. Stillinger and Thomas A. Weber},
  title        = {{S}tillinger-{W}eber ({SW}) {M}odel {D}river v005},
  doi          = {10.25950/934dca3e},
  howpublished = {OpenKIM, \url{https://doi.org/10.25950/934dca3e}},
  publisher    = {OpenKIM},
  year         = 2021,
}

@Article{SW_paper_1,
  author = {Stillinger, Frank H. and Weber, Thomas A.},
  doi = {10.1103/PhysRevB.31.5262},
  issue = {8},
  journal = {Physical Review B},
  month = {Apr},
  pages = {5262--5271},
  publisher = {American Physical Society},
  title = {Computer simulation of local order in condensed phases of silicon},
  volume = {31},
  year = {1985},
}

@Article{SW_paper_2,
  author = {Stillinger, Frank H. and Weber, Thomas A.},
  doi = {10.1103/PhysRevB.33.1451},
  issue = {2},
  journal = {Phys. Rev. B},
  month = {jan},
  numpages = {0},
  pages = {1451--1451},
  publisher = {American Physical Society},
  title = {Erratum: Computer simulation of local order in condensed phases of silicon [{P}hys. {R}ev. {B} 31, 5262 (1985)]},
  volume = {33},
  year = {1986},
}

@Misc{OpenKIM_EDIP_MO,
    author       = {Daniel S. Karls and Joao F. Justo and Martin Z. Bazant and Efthimios Kaxiras and Vasily V Bulatov and Sidney Yip},
    title        = {{EDIP} model for {S}i developed by {J}usto et al. (1998) v002},
    doi          = {10.25950/545ca247},
    howpublished = {OpenKIM, \url{https://doi.org/10.25950/545ca247}},
    publisher    = {OpenKIM},
    year         = 2018,
}

@Misc{OpenKIM_EDIP_MD,
    author       = {Daniel S. Karls and Joao F. Justo and Martin Z. Bazant and Efthimios Kaxiras and Vasily V Bulatov and Sidney Yip},
    title        = {{E}nvironment-{D}ependent {I}nteratomic {P}otential ({EDIP}) model driver v002},
    doi          = {10.25950/75c4686e},
    howpublished = {OpenKIM, \url{https://doi.org/10.25950/75c4686e}},
    publisher    = {OpenKIM},
    year         = 2018,
}

@Article{EDIP_paper,
    author = {Justo, Jo\~{a}o F. and Bazant, Martin Z. and Kaxiras, Efthimios and Bulatov, V. V. and Yip, Sidney},
    doi = {10.1103/PhysRevB.58.2539},
    issue = {5},
    journal = {Physical Review B},
    month = {Aug},
    pages = {2539--2550},
    publisher = {American Physical Society},
    title = {Interatomic potential for silicon defects and disordered phases},
    volume = {58},
    year = {1998},
}

@Misc{OpenKIM_elastic_constants_TE,
  author       = {Junhao Li and Ellad Tadmor},
  title        = {{E}lastic constants for diamond {S}i at zero temperature v001},
  howpublished = {OpenKIM, \url{https://openkim.org/cite/TE_507832142782_001}},
  publisher    = {OpenKIM},
  year         = 2019,
}

@Misc{OpenKIM_elastic_constants_TD,
  author       = {Ellad B. Tadmor and Junhao Li},
  title        = {{E}lastic constants for cubic crystals at zero temperature and pressure v006},
  doi          = {10.25950/5853fb8f},
  howpublished = {OpenKIM, \url{https://doi.org/10.25950/5853fb8f}},
  publisher    = {OpenKIM},
  year         = 2019,
}

@misc{information-matching_repo,
  author = {Kurniawan, Yonatan},
  title = {{I}nformation-matching {G}itHub {R}epository},
  year = {2024},
  publisher = {GitHub},
  journal = {GitHub repository},
  howpublished = {\url{https://github.com/yonatank93/information-matching}},
}

@misc{Kurniawan2024,
  author = "Yonatan Kurniawan",
  title = "{information\_matching\_examples\_dataset.tar.gz}",
  year = "2024",
  month = "10",
  url = "https://figshare.com/articles/dataset/information_matching_examples_dataset_tar_gz/27274029",
  doi = "10.6084/m9.figshare.27274029.v1"
}

@article{transtrum2015perspective,
  author = {Mark K. Transtrum and Benjamin B. Machta and Kevin S. Brown and Bryan C. Daniels and Christopher R. Myers and James P. Sethna},
  title = {Perspective: Sloppiness and Emergent Theories in Physics, Biology, and Beyond},
  doi = {10.1063/1.4923066},
  number = {1},
  pages = {010901},
  url = {https://doi.org/10.1063/1.4923066},
  volume = {143},
  date_added = {Wed Jul 19 23:03:34 2017},
  journal = {The Journal of Chemical Physics},
  year = {2015},
}

@article{casey2007optimal,
  author = {F.P. Casey and R.N. Gutenkunst and C.R. Myers and D. Baird and K.S. Brown and J.J. Waterfall and Q. Feng and R.A. Cerione and J.P. Sethna},
  title = {Optimal Experimental Design in an Epidermal Growth Factor Receptor Signalling and Down-Regulation Model},
  doi = {10.1049/iet-syb:20060065},
  number = {3},
  pages = {190-202},
  url = {https://doi.org/10.1049/iet-syb:20060065},
  volume = {1},
  date_added = {Wed Jul 19 23:54:30 2017},
  journal = {IET Systems Biology},
  year = {2007},
}

@article{transtrum2012optimal,
  author = {Mark K Transtrum and Peng Qiu},
  title = {Optimal Experiment Selection for Parameter Estimation in Biological Differential Equation Models},
  doi = {10.1186/1471-2105-13-181},
  number = {1},
  pages = {181},
  url = {https://doi.org/10.1186/1471-2105-13-181},
  volume = {13},
  date_added = {Wed Jul 19 23:14:26 2017},
  journal = {BMC Bioinformatics},
  year = {2012},
}

@misc{matpowerDescriptionCase39,
  author = {},
  title = {Power flow data for 39 bus {New England} system.},
  howpublished = {\url{https://matpower.org/docs/ref/matpower5.0/case39.html}},
  year = {},
  note = {[Accessed 25--06--2024]},
}

@book{LeSar_2013,
  place={Cambridge},
  title={Introduction to Computational Materials Science: Fundamentals to Applications},
  publisher={Cambridge University Press},
  author={LeSar, Richard},
  year={2013}
}

@article{saric2018data,
  title={Data-driven dynamic equivalents for power system areas from boundary measurements},
  author={Sari{\'c}, Andrija T and Transtrum, Mark T and Stankovi{\'c}, Aleksandar M},
  journal={IEEE Transactions on Power Systems},
  volume={34},
  number={1},
  pages={360--370},
  year={2018},
  publisher={IEEE}
}

@book{van2007parameter,
  title={Parameter estimation for scientists and engineers},
  author={Van den Bos, Adriaan},
  year={2007},
  publisher={John Wiley \& Sons}
}

@book{cramer1999mathematical,
  title={Mathematical methods of statistics},
  author={Cram{\'e}r, Harald},
  volume={26},
  year={1999},
  publisher={Princeton university press}
}

@misc{ieee39bus,
    author={Illinois Center for a Smarter Electric Grid},
    title={{IEEE 39-Bus System}},
    howpublished={\url{https://icseg.iti.illinois.edu/ieee-39-bus-system/}},
    note={Accessed: 2024--08--06}
}

@misc{ieee14bus,
    author={Illinois Center for a Smarter Electric Grid},
    title={{IEEE 14-Bus System}},
    howpublished={\url{https://icseg.iti.illinois.edu/ieee-14-bus-system/}},
    note={Accessed: 2024--08--12}
}

@Article{LAMMPS,
  author = "A. P. Thompson and H. M. Aktulga and R. Berger and 
     D. S. Bolintineanu and W. M. Brown and P. S. Crozier and
     P. J. in 't Veld and A. Kohlmeyer and S. G. Moore and T. D. Nguyen and
     R. Shan and M. J. Stevens and J. Tranchida and C. Trott and S. J. Plimpton",
  title = "{LAMMPS} - a flexible simulation tool for
     particle-based materials modeling at the 
     atomic, meso, and continuum scales",
  journal = "Comp. Phys. Comm.",
  volume =  "271",
  pages =   "108171",
  year =    "2022",
  doi = "10.1016/j.cpc.2021.108171"
}

@article{ase-paper,
  author={Ask Hjorth Larsen and Jens Jørgen Mortensen and Jakob Blomqvist and Ivano E Castelli and Rune Christensen and Marcin
Dułak and Jesper Friis and Michael N Groves and Bjørk Hammer and Cory Hargus and Eric D Hermes and Paul C Jennings and Peter
Bjerre Jensen and James Kermode and John R Kitchin and Esben Leonhard Kolsbjerg and Joseph Kubal and Kristen
Kaasbjerg and Steen Lysgaard and J\'{o}n Bergmann Maronsson and Tristan Maxson and Thomas Olsen and Lars Pastewka and Andrew
Peterson and Carsten Rostgaard and Jakob Schiøtz and Ole Schütt and Mikkel Strange and Kristian S Thygesen and Tejs
Vegge and Lasse Vilhelmsen and Michael Walter and Zhenhua Zeng and Karsten W Jacobsen},
  title={The atomic simulation environment—a Python library for working with atoms},
  journal={Journal of Physics: Condensed Matter},
  volume=29,
  number=27,
  pages=273002,
  url={http://stacks.iop.org/0953-8984/29/i=27/a=273002},
  year=2017
}

@article{diamond2016cvxpy,
  author  = {Steven Diamond and Stephen Boyd},
  title   = {{CVXPY}: {A} {P}ython-embedded modeling language for convex optimization},
  journal = {Journal of Machine Learning Research},
  year    = {2016},
  volume  = {17},
  number  = {83},
  pages   = {1--5},
}

@article{agrawal2018rewriting,
  author  = {Agrawal, Akshay and Verschueren, Robin and Diamond, Steven and Boyd, Stephen},
  title   = {A rewriting system for convex optimization problems},
  journal = {Journal of Control and Decision},
  year    = {2018},
  volume  = {5},
  number  = {1},
  pages   = {42--60},
}

@article{doi:10.1080/1055678031000118482,
    author    = "Yamashita, Makoto
                 and Fujisawa, Katsuki
                 and Kojima, Masakazu",
    title     = "Implementation and evaluation of SDPA 6.0 (Semidefinite Programming Algorithm 6.0)",
    journal   = "Optimization Methods and Software",
    volume    = "18",
    number    = "4",
    pages     = "491-505",
    year      = "2003",
    publisher = "Taylor & Francis",
    doi       = "10.1080/1055678031000118482",
    URL       = "https://doi.org/10.1080/1055678031000118482",
    eprint    = "https://doi.org/10.1080/1055678031000118482"
}

@Inbook{Yamashita2012,
    author    = "Yamashita, Makoto
                 and Fujisawa, Katsuki
                 and Fukuda, Mituhiro
                 and Kobayashi, Kazuhiro
                 and Nakata, Kazuhide
                 and Nakata, Maho",
    editor    = "Anjos, Miguel F.
                 and Lasserre, Jean B.",
    title           = "Latest Developments in the SDPA Family for Solving Large-Scale SDPs",
    bookTitle       = "Handbook on Semidefinite, Conic and Polynomial Optimization",
    year      = "2012",
    publisher = "Springer US",
    address   = "Boston, MA",
    pages     = "687--713",
    isbn      = "978-1-4614-0769-0",
    doi       = "10.1007/978-1-4614-0769-0_24",
    url       = "https://doi.org/10.1007/978-1-4614-0769-0_24"
}

@inproceedings{doi:10.1109/CACSD.2010.5612693,
    author    = "Nakata, Maho",
    booktitle = "2010 IEEE International Symposium on Computer-Aided Control System Design", 
    title     = "A numerical evaluation of highly accurate multiple-precision arithmetic version of semidefinite programming solver: SDPA-GMP, -QD and -DD.", 
    year      = "2010",
    volume    = "",
    number    = "",
    pages     = "29-34",
    doi       = "10.1109/CACSD.2010.5612693"
}

@article{Kim2011,
    author    = "Kim, Sunyoung
                 and Kojima, Masakazu
                 and Mevissen, Martin
                 and Yamashita, Makoto",
    title     = "Exploiting sparsity in linear and nonlinear matrix inequalities via positive semidefinite matrix completion",
    journal   = "Mathematical Programming",
    year      = "2011",
    month     = "Sep",
    day       = "01",
    volume    = "129",
    number    = "1",
    pages     = "33-68",
    issn      = "1436-4646",
    doi       = "10.1007/s10107-010-0402-6",
    url       = "https://doi.org/10.1007/s10107-010-0402-6"
}

@article{ocpb:16,
    author       = {Brendan O'Donoghue and Eric Chu and Neal Parikh and Stephen Boyd},
    title        = {Conic Optimization via Operator Splitting and Homogeneous Self-Dual Embedding},
    journal      = {Journal of Optimization Theory and Applications},
    month        = {June},
    year         = {2016},
    volume       = {169},
    number       = {3},
    pages        = {1042-1068},
    url          = {http://stanford.edu/~boyd/papers/scs.html},
}

@article{odonoghue:21,
    author       = {Brendan O'Donoghue},
    title        = {Operator Splitting for a Homogeneous Embedding of the Linear Complementarity Problem},
    journal      = {{SIAM} Journal on Optimization},
    month        = {August},
    year         = {2021},
    volume       = {31},
    issue        = {3},
    pages        = {1999-2023},
}

@article{aa2020,
  title={Globally Convergent {type--I} {A}nderson Acceleration for Non-Smooth Fixed-Point Iterations},
  author={Junzi Zhang and Brendan O'Donoghue and Stephen Boyd},
  journal={{SIAM} Journal on Optimization},
  volume={30},
  number={4},
  pages={3170--3197},
  year={2020}
}

@article{scip,
  author = {Bestuzheva, Ksenia and Besan\c{c}on, Mathieu and Chen, Wei-Kun and Chmiela, Antonia and Donkiewicz, Tim and van Doornmalen, Jasper and Eifler, Leon and Gaul, Oliver and Gamrath, Gerald and Gleixner, Ambros and Gottwald, Leona and Graczyk, Christoph and Halbig, Katrin and Hoen, Alexander and Hojny, Christopher and van der Hulst, Rolf and Koch, Thorsten and L\"{u}bbecke, Marco and Maher, Stephen J. and Matter, Frederic and M\"{u}hmer, Erik and M\"{u}ller, Benjamin and Pfetsch, Marc E. and Rehfeldt, Daniel and Schlein, Steffan and Schl\"{o}sser, Franziska and Serrano, Felipe and Shinano, Yuji and Sofranac, Boro and Turner, Mark and Vigerske, Stefan and Wegscheider, Fabian and Wellner, Philipp and Weninger, Dieter and Witzig, Jakob},
  title = {Enabling Research through the SCIP Optimization Suite 8.0},
  year = {2023},
  issue_date = {June 2023},
  publisher = {Association for Computing Machinery},
  address = {New York, NY, USA},
  volume = {49},
  number = {2},
  issn = {0098-3500},
  url = {https://doi.org/10.1145/3585516},
  doi = {10.1145/3585516},
  journal = {ACM Trans. Math. Softw.},
  month = {jun},
  articleno = {22},
  numpages = {21}
}

@article{scip_sdp,
  author = {Tristan Gally, Marc E. Pfetsch and Stefan Ulbrich},
  title = {A framework for solving mixed-integer semidefinite programs},
  journal = {Optimization Methods and Software},
  volume = {33},
  number = {3},
  pages = {594--632},
  year = {2018},
  publisher = {Taylor \& Francis},
  doi = {10.1080/10556788.2017.1322081},
  URL = {https://doi.org/10.1080/10556788.2017.1322081},
  eprint = {https://doi.org/10.1080/10556788.2017.1322081}
}

@article{artrith_high-dimensional_2012_edited,
	title = {High-dimensional neural network potentials for metal surfaces: {A} prototype study for copper},
	volume = {85},
	shorttitle = {High-dimensional neural network potentials for metal surfaces},
	url = {https://link.aps.org/doi/10.1103/PhysRevB.85.045439},
	doi = {10.1103/PhysRevB.85.045439},
	number = {4},
	urldate = {2024-02-06},
	journal = {Physical Review B},
	author = {Artrith, Nongnuch and Behler, J\"{o}rg},
	month = jan,
	year = {2012},
	note = {Publisher: American Physical Society},
	pages = {045439},
}

@article{csanyi_learn_2004_edited,
	title = {``{Learn} on the {Fly}'': {A} {Hybrid} {Classical} and {Quantum}-{Mechanical} {Molecular} {Dynamics} {Simulation}},
	volume = {93},
	shorttitle = {``{Learn} on the {Fly}''},
	url = {https://link.aps.org/doi/10.1103/PhysRevLett.93.175503},
	doi = {10.1103/PhysRevLett.93.175503},
	number = {17},
	urldate = {2024-02-06},
	journal = {Physical Review Letters},
	author = {Cs\'{a}nyi, G\'{a}bor and Albaret, T. and Payne, M. C. and De Vita, A.},
	month = oct,
	year = {2004},
	note = {Publisher: American Physical Society},
	pages = {175503},
}

@article{wen_force-matching_2017_edited,
	title = {A force-matching {Stillinger}-{Weber} potential for {MoS} $_{\textrm{2}}$ : {Parameterization} and {Fisher} information theory based sensitivity analysis},
	volume = {122},
	issn = {0021-8979, 1089-7550},
	shorttitle = {A force-matching {Stillinger}-{Weber} potential for {MoS} $_{\textrm{2}}$},
	url = {http://aip.scitation.org/doi/10.1063/1.5007842},
	doi = {10.1063/1.5007842},
	language = {en},
	number = {24},
	urldate = {2020-02-21},
	journal = {Journal of Applied Physics},
	author = {Wen, Mingjian and Shirodkar, Sharmila N. and Plech\'{a}\v{c}, Petr and Kaxiras, Efthimios and Elliott, Ryan S. and Tadmor, Ellad B.},
	month = dec,
	year = {2017},
	pages = {244301},
}

@book{
tolstoy:1993:matched,
   Author = {Tolstoy, Alexandra},
   Title = {Matched field processing for underwater acoustics},
   Publisher = {World Scientific},
      Year = {1993} }

@article{chakrabarti_optimal_2008_edited,
	title = {Optimal {Placement} of {Phasor} {Measurement} {Units} for {Power} {System} {Observability}},
	volume = {23},
	issn = {1558-0679},
	url = {https://ieeexplore.ieee.org/document/4519389},
	doi = {10.1109/TPWRS.2008.922621},
	number = {3},
	urldate = {2023-10-31},
	journal = {IEEE Transactions on Power Systems},
	author = {Chakrabarti, Saikat and Kyriakides, Elias},
	month = aug,
	year = {2008},
	pages = {1433--1440},
}

@article{aminifar_contingency-constrained_2010_edited,
	title = {Contingency-{Constrained} {PMU} {Placement} in {Power} {Networks}},
	volume = {25},
	issn = {1558-0679},
	url = {https://ieeexplore.ieee.org/document/5357471},
	doi = {10.1109/TPWRS.2009.2036470},
	number = {1},
	urldate = {2023-10-31},
	journal = {IEEE Transactions on Power Systems},
	author = {Aminifar, Farrokh and Khodaei, Amin and Fotuhi-Firuzabad, Mahmud and Shahidehpour, Mohammad},
	month = feb,
	year = {2010},
	pages = {516--523},
}

@article{kim_uncertainty_2022_edited,
	title = {Uncertainty {Assessment}-{Based} {Active} {Learning} for {Reliable} {Fire} {Detection} {Systems}},
	volume = {10},
	issn = {2169-3536},
	url = {https://ieeexplore.ieee.org/abstract/document/9829744},
	doi = {10.1109/ACCESS.2022.3190852},
	urldate = {2024-04-01},
	journal = {IEEE Access},
	author = {Kim, Young-Jin and Kim, Won-Tae},
	year = {2022},
	pages = {74722--74732},
}

@article{zhao_power_2019_edited,
	title = {Power {Grid} {Partitioning} {Based} on {Functional} {Community} {Structure}},
	volume = {7},
	issn = {2169-3536},
	url = {https://ieeexplore-ieee-org.byu.idm.oclc.org/document/8878124},
	doi = {10.1109/ACCESS.2019.2948606},
	urldate = {2024-07-02},
	journal = {IEEE Access},
	author = {Zhao, Chuanzhi and Zhao, Jintang and Wu, Chunchao and Wang, Xiaoliang and Xue, Fei and Lu, Shaofeng},
	year = {2019},
	pages = {152624--152634},
}

@article{baldwin_power_1993,
	title = {Power system observability with minimal phasor measurement placement},
	volume = {8},
	issn = {1558-0679},
	url = {https://ieeexplore.ieee.org/document/260810},
	doi = {10.1109/59.260810},
	number = {2},
	urldate = {2023-10-31},
	journal = {IEEE Transactions on Power Systems},
	author = {Baldwin, T.L. and Mili, L. and Boisen, M.B. and Adapa, R.},
	month = may,
	year = {1993},
	note = {Conference Name: IEEE Transactions on Power Systems},
	pages = {707--715},
}

@article{milosevic_nondominated_2003,
	title = {Nondominated sorting genetic algorithm for optimal phasor measurement placement},
	volume = {18},
	issn = {1558-0679},
	url = {https://ieeexplore.ieee.org/document/1178767},
	doi = {10.1109/TPWRS.2002.807064},
	number = {1},
	urldate = {2026-01-09},
	journal = {IEEE Transactions on Power Systems},
	author = {Milosevic, B. and Begovic, M.},
	month = feb,
	year = {2003},
	pages = {69--75},
}

@article{soudi_optimal_1999,
	title = {Optimal distribution protection design: quality of solution and computational analysis},
	volume = {21},
	issn = {0142-0615},
	shorttitle = {Optimal distribution protection design},
	url = {https://www.sciencedirect.com/science/article/pii/S0142061598000520},
	doi = {10.1016/S0142-0615(98)00052-0},
	number = {5},
	urldate = {2025-01-22},
	journal = {International Journal of Electrical Power \& Energy Systems},
	author = {Soudi, F and Tomsovic, K},
	month = jun,
	year = {1999},
	pages = {327--335},
}

@inproceedings{tidwell_designing_2019,
	title = {Designing {Linear} {FM} {Active} {Sonar} {Waveforms} for {Continuous} {Line} {Source} {Transducers} to {Maximize} the {Fisher} {Information} at a {Desired} {Bearing}},
	url = {https://ieeexplore.ieee.org/document/8751647},
	doi = {10.1109/SSPD.2019.8751647},
	urldate = {2025-01-22},
	booktitle = {2019 {Sensor} {Signal} {Processing} for {Defence} {Conference} ({SSPD})},
	author = {Tidwell, Matthew D. and Buck, John R.},
	month = may,
	year = {2019},
	pages = {1--5},
}

@inproceedings{buck_information_2002,
	title = {Information theoretic bounds on source localization performance},
	url = {https://ieeexplore.ieee.org/document/1191025},
	doi = {10.1109/SAM.2002.1191025},
	urldate = {2025-01-22},
	booktitle = {Sensor {Array} and {Multichannel} {Signal} {Processing} {Workshop} {Proceedings}, 2002},
	author = {Buck, J.R.},
	month = aug,
	year = {2002},
	pages = {184--188},
}

@article{tavazza_uncertainty_2021,
	title = {Uncertainty {Prediction} for {Machine} {Learning} {Models} of {Material} {Properties}},
	volume = {6},
	url = {https://doi.org/10.1021/acsomega.1c03752},
	doi = {10.1021/acsomega.1c03752},
	number = {48},
	urldate = {2025-01-22},
	journal = {ACS Omega},
	author = {Tavazza, Francesca and DeCost, Brian and Choudhary, Kamal},
	month = dec,
	year = {2021},
	note = {Publisher: American Chemical Society},
	pages = {32431--32440},
}

@article{mortenson_accurate_2023,
	title = {Accurate {Broadband} {Gradient} {Estimates} {Enable} {Local} {Sensitivity} {Analysis} of {Ocean} {Acoustic} {Models}},
	volume = {31},
	issn = {2591-7285},
	url = {https://www.worldscientific.com/doi/10.1142/S2591728522500153},
	doi = {10.1142/S2591728522500153},
	number = {02},
	urldate = {2024-08-06},
	journal = {Journal of Theoretical and Computational Acoustics},
	author = {Mortenson, Michael C. and Neilsen, Tracianne B. and Transtrum, Mark K. and Knobles, David P.},
	month = jun,
	year = {2023},
	note = {Publisher: World Scientific Publishing Co.},
	pages = {2250015},
}

@article{shmueli_predictive_2011,
	title = {Predictive {Analytics} in {Information} {Systems} {Research}},
	volume = {35},
	issn = {0276-7783},
	url = {https://www.jstor.org/stable/23042796},
	doi = {10.2307/23042796},
	number = {3},
	urldate = {2024-07-16},
	journal = {MIS Quarterly},
	author = {Shmueli, Galit and Koppius, Otto R.},
	year = {2011},
	note = {Publisher: Management Information Systems Research Center, University of Minnesota},
	pages = {553--572},
}

@misc{wen_dataset_2024,
	title = {Dataset of {MoS2} monolayer from {AIMD} trajectory},
	url = {https://zenodo.org/records/12553773},
	doi = {10.5281/zenodo.12553773},
	urldate = {2024-06-26},
	publisher = {Zenodo},
	author = {Wen, Mingjian},
	month = jun,
	year = {2024},
}

@misc{brouwer_underlying_2018,
	title = {The underlying connections between identifiability, active subspaces, and parameter space dimension reduction},
	url = {http://arxiv.org/abs/1802.05641},
	doi = {10.48550/arXiv.1802.05641},
	urldate = {2024-03-28},
	publisher = {arXiv},
	author = {Brouwer, Andrew F. and Eisenberg, Marisa C.},
	month = feb,
	year = {2018},
	note = {arXiv:1802.05641 [math]},
}

@inproceedings{transtrum_simultaneous_2018,
	title = {Simultaneous {Global} {Identification} of {Dynamic} and {Network} {Parameters} in {Transient} {Stability} {Studies}},
	url = {https://ieeexplore.ieee.org/abstract/document/8586586},
	doi = {10.1109/PESGM.2018.8586586},
	urldate = {2024-02-08},
	booktitle = {2018 {IEEE} {Power} \& {Energy} {Society} {General} {Meeting} ({PESGM})},
	author = {Transtrum, Mark K. and Francis, Benjamin L. and Saric, Andrija T. and Stankovic, Aleksandar M.},
	month = aug,
	year = {2018},
	note = {ISSN: 1944-9933},
	pages = {1--5},
}

@article{wang_active-learning_2021,
	title = {Active-{Learning} {Approaches} for {Landslide} {Mapping} {Using} {Support} {Vector} {Machines}},
	volume = {13},
	copyright = {http://creativecommons.org/licenses/by/3.0/},
	issn = {2072-4292},
	url = {https://www.mdpi.com/2072-4292/13/13/2588},
	doi = {10.3390/rs13132588},
	language = {en},
	number = {13},
	urldate = {2024-04-01},
	journal = {Remote Sensing},
	author = {Wang, Zhihao and Brenning, Alexander},
	month = jan,
	year = {2021},
	note = {Number: 13
Publisher: Multidisciplinary Digital Publishing Institute},
	pages = {2588},
}

@article{streiner_precision_2006,
	title = {“{Precision}” and “{Accuracy}”: {Two} {Terms} {That} {Are} {Neither}},
	volume = {59},
	issn = {0895-4356},
	shorttitle = {“{Precision}” and “{Accuracy}”},
	url = {https://www.sciencedirect.com/science/article/pii/S0895435605003409},
	doi = {10.1016/j.jclinepi.2005.09.005},
	number = {4},
	urldate = {2024-04-01},
	journal = {Journal of Clinical Epidemiology},
	author = {Streiner, David L. and Norman, Geoffrey R.},
	month = apr,
	year = {2006},
	pages = {327--330},
}

@article{transtrum_geometry_2011,
	title = {Geometry of nonlinear least squares with applications to sloppy models and optimization},
	volume = {83},
	url = {https://link.aps.org/doi/10.1103/PhysRevE.83.036701},
	doi = {10.1103/PhysRevE.83.036701},
	number = {3},
	urldate = {2020-05-29},
	journal = {Physical Review E},
	author = {Transtrum, Mark K. and Machta, Benjamin B. and Sethna, James P.},
	month = mar,
	year = {2011},
	note = {Publisher: American Physical Society},
	pages = {036701},
}

@inproceedings{alizadeh_survey_2021,
	address = {Cham},
	title = {Survey on {Recent} {Active} {Learning} {Methods} for {Deep} {Learning}},
	isbn = {978-3-030-69984-0},
	doi = {10.1007/978-3-030-69984-0_43},
	language = {en},
	booktitle = {Advances in {Parallel} \& {Distributed} {Processing}, and {Applications}},
	publisher = {Springer International Publishing},
	author = {Alizadeh, Azar and Tavallali, Pooya and Khosravi, Mohammad R. and Singhal, Mukesh},
	editor = {Arabnia, Hamid R. and Deligiannidis, Leonidas and Grimaila, Michael R. and Hodson, Douglas D. and Joe, Kazuki and Sekijima, Masakazu and Tinetti, Fernando G.},
	year = {2021},
	pages = {609--617},
}

@article{kurniawan_bayesian_2022,
	title = {Bayesian, frequentist, and information geometric approaches to parametric uncertainty quantification of classical empirical interatomic potentials},
	volume = {156},
	copyright = {All rights reserved},
	issn = {0021-9606},
	url = {https://aip.scitation.org/doi/full/10.1063/5.0084988},
	doi = {10.1063/5.0084988},
	number = {21},
	urldate = {2023-03-28},
	journal = {The Journal of Chemical Physics},
	author = {Kurniawan, Yonatan and Petrie, Cody L. and Williams, Kinamo J. and Transtrum, Mark K. and Tadmor, Ellad B. and Elliott, Ryan S. and Karls, Daniel S. and Wen, Mingjian},
	month = jun,
	year = {2022},
	note = {Publisher: American Institute of Physics},
	pages = {214103},
}

@article{quinn_information_2022,
	title = {Information geometry for multiparameter models: new perspectives on the origin of simplicity},
	volume = {86},
	issn = {0034-4885},
	shorttitle = {Information geometry for multiparameter models},
	url = {https://dx.doi.org/10.1088/1361-6633/aca6f8},
	doi = {10.1088/1361-6633/aca6f8},
	language = {en},
	number = {3},
	urldate = {2024-02-12},
	journal = {Reports on Progress in Physics},
	author = {Quinn, Katherine N. and Abbott, Michael C. and Transtrum, Mark K. and Machta, Benjamin B. and Sethna, James P.},
	month = dec,
	year = {2022},
	note = {Publisher: IOP Publishing},
	pages = {035901},
}

@article{balsa-canto_computing_2008,
	title = {Computing {Optimal} {Dynamic} {Experiments} for {Model} {Calibration} in {Predictive} {Microbiology}},
	volume = {31},
	copyright = {© 2008, The Author(s)},
	issn = {1745-4530},
	url = {https://onlinelibrary.wiley.com/doi/abs/10.1111/j.1745-4530.2007.00147.x},
	doi = {10.1111/j.1745-4530.2007.00147.x},
	language = {en},
	number = {2},
	urldate = {2024-02-19},
	journal = {Journal of Food Process Engineering},
	author = {Balsa-Canto, E. and Alonso, A.a. and Banga, J.r.},
	year = {2008},
	note = {\_eprint: https://onlinelibrary.wiley.com/doi/pdf/10.1111/j.1745-4530.2007.00147.x},
	pages = {186--206},
}

@article{morgan_e-optimality_2011,
	title = {E-{Optimality} in {Treatment} versus {Control} {Experiments}},
	volume = {5},
	issn = {1559-8608},
	url = {https://doi.org/10.1080/15598608.2011.10412053},
	doi = {10.1080/15598608.2011.10412053},
	number = {1},
	urldate = {2024-02-19},
	journal = {Journal of Statistical Theory and Practice},
	author = {Morgan, J. P. and Wang, Xiaowei},
	month = mar,
	year = {2011},
	note = {Publisher: Taylor \& Francis
\_eprint: https://doi.org/10.1080/15598608.2011.10412053},
	pages = {99--107},
}

@article{chow_e-optimality_1997,
	title = {\textit{{E}}-{Optimality} for {Regression} {Designs} of {Supplementary} {Experiments}},
	volume = {214},
	issn = {0022-247X},
	url = {https://www.sciencedirect.com/science/article/pii/S0022247X97956067},
	doi = {10.1006/jmaa.1997.5606},
	number = {1},
	urldate = {2024-02-19},
	journal = {Journal of Mathematical Analysis and Applications},
	author = {Chow, King Leung},
	month = oct,
	year = {1997},
	pages = {207--218},
}

@article{andere-rendon_design_1997,
	title = {Design of {Mixture} {Experiments} {Using} {Bayesian} {D}-{Optimality}},
	volume = {29},
	issn = {0022-4065},
	url = {https://doi.org/10.1080/00224065.1997.11979796},
	doi = {10.1080/00224065.1997.11979796},
	number = {4},
	urldate = {2024-02-19},
	journal = {Journal of Quality Technology},
	author = {Andere-Rendon, Jose and Montgomery, Douglas C. and Rollier, Dwayne A.},
	month = oct,
	year = {1997},
	note = {Publisher: Taylor \& Francis
\_eprint: https://doi.org/10.1080/00224065.1997.11979796},
	pages = {451--463},
}

@article{butler_approximate_2008,
	title = {On an {Approximate} {Optimality} {Criterion} for the {Design} of {Field} {Experiments} {Under} {Spatial} {Dependence}},
	volume = {50},
	copyright = {© 2008 Australian Statistical Publishing Association Inc.},
	issn = {1467-842X},
	url = {https://onlinelibrary.wiley.com/doi/abs/10.1111/j.1467-842X.2008.00518.x},
	doi = {10.1111/j.1467-842X.2008.00518.x},
	language = {en},
	number = {4},
	urldate = {2024-02-19},
	journal = {Australian \& New Zealand Journal of Statistics},
	author = {Butler, David G. and Eccleston, John A. and Cullis, Brian R.},
	year = {2008},
	note = {\_eprint: https://onlinelibrary.wiley.com/doi/pdf/10.1111/j.1467-842X.2008.00518.x},
	pages = {295--307},
}

@article{jacroux_-optimality_1989,
	title = {The {A}-{Optimality} of {Block} {Designs} for {Comparing} {Test} {Treatments} with a {Control}},
	volume = {84},
	issn = {0162-1459},
	url = {https://doi.org/10.1080/01621459.1989.10478771},
	doi = {10.1080/01621459.1989.10478771},
	number = {405},
	urldate = {2024-02-19},
	journal = {Journal of the American Statistical Association},
	author = {Jacroux, Mike},
	month = mar,
	year = {1989},
	note = {Publisher: Taylor \& Francis
\_eprint: https://doi.org/10.1080/01621459.1989.10478771},
	pages = {310--317},
}

@article{jones_-optimal_2021,
	title = {A-optimal versus {D}-optimal design of screening experiments},
	volume = {53},
	issn = {0022-4065},
	url = {https://doi.org/10.1080/00224065.2020.1757391},
	doi = {10.1080/00224065.2020.1757391},
	number = {4},
	urldate = {2024-02-19},
	journal = {Journal of Quality Technology},
	author = {Jones, Bradley and Allen-Moyer, Katherine and Goos, Peter},
	month = aug,
	year = {2021},
	note = {Publisher: Taylor \& Francis
\_eprint: https://doi.org/10.1080/00224065.2020.1757391},
	pages = {369--382},
}

@article{leardi_experimental_2009,
	series = {Fundamental and {Applied} {Analytical} {Science}. {A} {Special} {Issue} {In} {Honour} of {Alan} {Townshend}.},
	title = {Experimental design in chemistry: {A} tutorial},
	volume = {652},
	issn = {0003-2670},
	shorttitle = {Experimental design in chemistry},
	url = {https://www.sciencedirect.com/science/article/pii/S0003267009008058},
	doi = {10.1016/j.aca.2009.06.015},
	number = {1},
	urldate = {2024-02-14},
	journal = {Analytica Chimica Acta},
	author = {Leardi, Riccardo},
	month = oct,
	year = {2009},
	pages = {161--172},
}

@article{dosso_array_2006,
	title = {Array element localization accuracy and survey design},
	volume = {34},
	copyright = {Copyright (c)},
	issn = {2291-1391},
	url = {https://jcaa.caa-aca.ca/index.php/jcaa/article/view/1851},
	language = {en},
	number = {4},
	urldate = {2024-02-06},
	journal = {Canadian Acoustics},
	author = {Dosso, Stan E. and Ebbeson, Gordon R.},
	month = dec,
	year = {2006},
	note = {Number: 4},
	pages = {3--13},
}

@article{dosso_optimal_1999,
	title = {Optimal array element localization},
	volume = {106},
	issn = {0001-4966},
	url = {https://doi.org/10.1121/1.428198},
	doi = {10.1121/1.428198},
	number = {6},
	urldate = {2024-02-05},
	journal = {The Journal of the Acoustical Society of America},
	author = {Dosso, Stan E. and Sotirin, Barbara J.},
	month = dec,
	year = {1999},
	pages = {3445--3459},
}

@article{barlee_array_2002,
	title = {Array element localization of a bottom moored hydrophone array},
	volume = {30},
	copyright = {Copyright (c)},
	issn = {2291-1391},
	url = {https://jcaa.caa-aca.ca/index.php/jcaa/article/view/1511},
	language = {en},
	number = {4},
	urldate = {2024-01-14},
	journal = {Canadian Acoustics},
	author = {Barlee, Matthew and Dosso, Stan and Schey, Philip},
	month = dec,
	year = {2002},
	note = {Number: 4},
	pages = {3--14},
}

@article{wood_optimisation_2003,
	title = {Optimisation of hydrophone placement: a dynamical systems approach},
	volume = {14},
	issn = {1469-4425, 0956-7925},
	shorttitle = {Optimisation of hydrophone placement},
	url = {https://www.cambridge.org/core/journals/european-journal-of-applied-mathematics/article/optimisation-of-hydrophone-placement-a-dynamical-systems-approach/F0267B3F44A783568C1E60C592B097A3},
	doi = {10.1017/S0956792503005175},
	language = {en},
	number = {4},
	urldate = {2024-01-14},
	journal = {European Journal of Applied Mathematics},
	author = {Wood, D. A. and Allwright, D. J.},
	month = aug,
	year = {2003},
	note = {Publisher: Cambridge University Press},
	pages = {369--386},
}

@inproceedings{hajian_optimal_2007,
	title = {Optimal {Placement} of {Phasor} {Measurement} {Units}: {Particle} {Swarm} {Optimization} {Approach}},
	shorttitle = {Optimal {Placement} of {Phasor} {Measurement} {Units}},
	url = {https://ieeexplore.ieee.org/document/4441610},
	doi = {10.1109/ISAP.2007.4441610},
	urldate = {2023-10-31},
	booktitle = {2007 {International} {Conference} on {Intelligent} {Systems} {Applications} to {Power} {Systems}},
	author = {Hajian, M. and Ranjbar, A. M. and Amraee, T. and Shirani, A. R.},
	month = nov,
	year = {2007},
	pages = {1--6},
}

@article{ercolessi_interatomic_1994,
	title = {Interatomic {Potentials} from {First}-{Principles} {Calculations}: {The} {Force}-{Matching} {Method}},
	volume = {26},
	issn = {0295-5075},
	shorttitle = {Interatomic {Potentials} from {First}-{Principles} {Calculations}},
	url = {https://dx.doi.org/10.1209/0295-5075/26/8/005},
	doi = {10.1209/0295-5075/26/8/005},
	language = {en},
	number = {8},
	urldate = {2022-12-31},
	journal = {Europhysics Letters},
	author = {Ercolessi, F. and Adams, J. B.},
	month = jun,
	year = {1994},
	pages = {583},
}

@article{westwood_normal_1996,
	title = {A normal mode model for acousto‐elastic ocean environments},
	volume = {100},
	issn = {0001-4966},
	url = {https://doi.org/10.1121/1.417226},
	doi = {10.1121/1.417226},
	number = {6},
	urldate = {2023-09-14},
	journal = {The Journal of the Acoustical Society of America},
	author = {Westwood, Evan K. and Tindle, C. T. and Chapman, N. R.},
	month = dec,
	year = {1996},
	pages = {3631--3645},
}

@article{peng_optimal_2006,
	title = {Optimal {PMU} placement for full network observability using {Tabu} search algorithm},
	volume = {28},
	issn = {0142-0615},
	url = {https://www.sciencedirect.com/science/article/pii/S0142061505001419},
	doi = {10.1016/j.ijepes.2005.05.005},
	number = {4},
	urldate = {2023-10-31},
	journal = {International Journal of Electrical Power \& Energy Systems},
	author = {Peng, Jiangnan and Sun, Yuanzhang and Wang, H. F.},
	month = may,
	year = {2006},
	pages = {223--231},
}

@article{gubaev_accelerating_2019,
	title = {Accelerating high-throughput searches for new alloys with active learning of interatomic potentials},
	volume = {156},
	issn = {0927-0256},
	url = {https://www.sciencedirect.com/science/article/pii/S0927025618306372},
	doi = {10.1016/j.commatsci.2018.09.031},
	urldate = {2023-10-19},
	journal = {Computational Materials Science},
	author = {Gubaev, Konstantin and Podryabinkin, Evgeny V. and Hart, Gus L. W. and Shapeev, Alexander V.},
	month = jan,
	year = {2019},
	pages = {148--156},
}

@inproceedings{yuill_optimal_2011,
	title = {Optimal {PMU} placement: {A} comprehensive literature review},
	shorttitle = {Optimal {PMU} placement},
	doi = {10.1109/PES.2011.6039376},
	booktitle = {2011 {IEEE} {Power} and {Energy} {Society} {General} {Meeting}},
	author = {Yuill, William and Edwards, A. and Chowdhury, S. and Chowdhury, S. P.},
	month = jul,
	year = {2011},
	note = {ISSN: 1944-9925},
	pages = {1--8},
}

@article{podryabinkin_active_2017,
	title = {Active learning of linearly parametrized interatomic potentials},
	volume = {140},
	issn = {0927-0256},
	url = {https://www.sciencedirect.com/science/article/pii/S0927025617304536},
	doi = {10.1016/j.commatsci.2017.08.031},
	language = {en},
	urldate = {2023-03-20},
	journal = {Computational Materials Science},
	author = {Podryabinkin, Evgeny V. and Shapeev, Alexander V.},
	month = dec,
	year = {2017},
	pages = {171--180},
}

\end{document}


\title{Supplementary Material: An information-matching approach to optimal experimental design and active learning}
\author{Yonatan Kurniawan}
\affiliation{Brigham Young University, Provo, UT 84602, USA}
\author{Tracianne B.~Neilsen}
\affiliation{Brigham Young University, Provo, UT 84602, USA}
\author{Benjamin L.~Francis}
\affiliation{Achilles Heel Technologies, Orem, UT 84097, USA}
\author{Alex M.~Stankovic}
\affiliation{SLAC National Accelerator Laboratory, Menlo Park, CA, USA}
\author{Mingjian Wen}
\affiliation{University of Electronic Science and Technology of China, Chengdu, 611731, China}
\author{Ilia Nikiforov}
\author{Ellad B.~Tadmor}
\affiliation{University of Minnesota, Minneapolis, MN 55455, USA}
\author{Vasily V.~Bulatov}
\author{Vincenzo Lordi}
\affiliation{Lawrence Livermore National Laboratory}
\author{Mark K.~Transtrum}
\email{mktranstrum@byu.edu}
\affiliation{Brigham Young University, Provo, UT 84602, USA}
\affiliation{Achilles Heel Technologies, Orem, UT 84097, USA}
\affiliation{SLAC National Accelerator Laboratory, Menlo Park, CA, USA}

\maketitle

\section{Least-squares regression and uncertainty propagation}
\label{sec:least-squares}

Consider a dataset consisting of $M$ observations $\{\bfp_m\}_{m=1}^M$ taken at input values $\{\bfx_m\}_{m=1}^M$.
The data is defined as the set of pairs $\{\bfx_m, \bfp_m\}_{m=1}^M$ for input values $\bfx_m$ and the corresponding labels $\bfp_m$.
We model the relationship between the input $\bfx_m$ and the label $\bfp_m$ using a model $\bff(\bfth; \bfx_m)$ parametrized by $\bfth$.
We further assume that the model can reproduce the observation within some additive random noise $\epsilon_m$,
\begin{equation*}
    \bfp_m = \bff(\bfth; \bfx_m) + \epsilon_m.
\end{equation*}
Motivated by the central limit theorem, it is common to assume that the noise follows a Gaussian distribution with mean zero and variance $\sigma_m^2$, i.e., $\epsilon_m \sim \mathcal{N}(0, \sigma_m^2)$.
This is equivalent to treating the label $\bfp_m$ as a Gaussian random variable, where the model prediction $\bff(\bfth; \bfx_m)$ serves as the mean and $\sigma_m^2$ as the variance of the label, providing a more interpretable measure of uncertainty.

Assuming independent data, the joint likelihood of the model given the observed labels is given by
\begin{equation}
    \label{eq:gaussian_likelihood}
    L(\bfth | \bfp) \propto \exp \left( -\frac{1}{2} \sum_{m=1}^M \frac{\lVert \bfp_m - \bff(\bfth; \bfx_m) \rVert_2^2}{\sigma_m^2} \right).
\end{equation}
The maximum likelihood estimator (MLE) is obtained by maximizing Eq.~\eqref{eq:gaussian_likelihood}.
Furthermore, the negative log-likelihood forms a commonly used weighted least-squares loss function (up to an additive constant),
\begin{equation}
    \label{eq:least-squares_loss}
    \ell(\bfth) = \frac{1}{2} \sum_{m=1}^M w_m \lVert \bfp_m - \bff(\bfth; \bfx_m) \rVert_2^2,
\end{equation}
where $w_m = 1/\sigma_m^2$ acts as the weight for each datum.
Since the logarithmic transformation is monotonic, then maximizing Eq.~\eqref{eq:gaussian_likelihood} is equivalent to a familiar least-squares regression that minimizes the loss function Eq.~\eqref{eq:least-squares_loss}.

The uncertainty inherent in the labels of a regression model propagates to uncertainty in the estimated parameters.
A fundamental object in parametric uncertainty quantification is the Fisher Information Matrix (FIM), which serves as a measure of the information content in the data about the model parameters.
Additionally, the FIM sets a lower bound on the covariance of the inferred parameters, known as the Cram\'{e}r-Rao bound\cite{streiner_precision_2006}.
This bound provides important insights into the precision of the estimated model parameters.

The FIM is defined as the expectation value of the Hessian of the log-likelihood with respect to the distributions of the labels.
Assuming independent data, the joint likelihood of the model given the observed labels is a product of the likelihood function given individual labels,
\begin{equation*}
    L(\bfth | \bfp) = \prod_{m=1}^M L(\bfth | \bfp_m).
\end{equation*}
Then, the FIM is given by
\begin{equation}
    \label{eq:fim}
	\CalI(\bfth) = \expectation{-\frac{\partial^2 \log L(\bfth | \bfp)}{\partial \bfth^2}}_\bfP
	= \sum_{m=1}^M \expectation{-\frac{\partial^2 \log L(\bfth | \bfp_m)}{\partial \bfth^2}}_{\bfP_m}
	= \sum_{m=1}^M \CalI_m (\bfth),
\end{equation}
where $\expectation{\cdot}_\bfP$ and $\expectation{\cdot}_{\bfP_m}$ denote the expectation value over the joint probability of the entire labels and over the probability of a single label $\bfp_m$, respectively, and $\CalI_m$ is the FIM for datum $\bfx_m$.
Equation~\eqref{eq:fim} highlights a generic, fundamental property of the FIM that the expected information in the entire dataset is the accumulation of information from each independent datum.

Restricting to a weighted least-squares problem, the negative log-likelihood is given by Eq.~\eqref{eq:least-squares_loss} and the FIM is given as
\begin{equation}
    \label{eq:fim_derivation}
    \begin{aligned}
      \CalI(\bfth) &= \expectation{-\frac{\partial^2 \log L(\bfth | \bfp)}{\partial \bfth^2}}_\bfP \\
			  &= \sum_{m=1}^M w_m \expectation{-\frac{1}{2}
			    \frac{\partial^2}{\partial \bfth^2} \lVert \bfp_m - \bff(\bfth; \bfx_m) \rVert_2^2
			    }_{\bfP_m} \\
			  &= \sum_{m=1}^M w_m J_{\bff}^T(\bfth; \bfx_m) J_{\bff}(\bfth; \bfx_m),
    \end{aligned}
\end{equation}
where we have set $\expectation{\bfp_m - \bff(\bfth; \bfx_m)}_{\bfP_m} = 0$ from the assumption of the noise and denote $J_{\bff}(\bfth; \bfx_m)$ as the Jacobian matrix of $\bff(\bfth; \bfx_m)$ for input $\bfx_m$ with respect to the parameters $\bfth$ \cite{transtrum_geometry_2011}.
The elements of $J_{\bff}(\bfth; \bfx_m)$ are calculated by
\begin{equation}
    \label{eq:jacobian}
    (J_{\bff})_{ij} (\bfth; \bfx_m) = \frac{\partial \bff_i (\bfth; \bfx_m)}{\partial \bfth_j},
\end{equation}
where $\bff_i(\bfth; \bfx_m)$ is the $i$-th element of the model output $\bff(\bfth; \bfx_m)$ and $\bfth_j$ is the $j$-th parameter.
With a slight abuse of notation and to overshadow the formulation used in the information-matching approach, we factor out the inverse data variance, i.e., weight, from the FIM and define the FIM for the $m$-th data with a unit data variance as
\begin{equation}
    \label{eq:fim_candidate}
    \CalI_m(\bfth) = J_{\bff}^T(\bfth; \bfx_m) J_{\bff}(\bfth; \bfx_m).
\end{equation}
Thus, the FIM for a weighted least-squares problem can be written as
\begin{equation}
    \label{eq:fim_least-squares}
    \CalI (\bfth) = \sum_{m=1}^M w_m \CalI_m (\bfth).
\end{equation}

The utility of a mathematical model often extends beyond parameter inference and into the realm of making new predictions of some quantities of interest (QoIs).
When making these predictions, the uncertainty associated with the parameters is further propagated to the QoIs.
The uncertainty of the QoIs directly impacts the reliability and credibility of the model predictions.

We denote the QoIs for input $\bfy$ as $\bfq$ and the model that approximates it as $\bfg(\bfth; \bfy)$.
The input $\bfy$ to the QoIs $\bfq$ is analogous to the input $\bfx_m$ for the training label $\bfp_m$, although $\bfy$ may also act as a control parameter for the QoIs or a discrete index to distinguish between different QoI values.
Then, consider the parameters $\bfth = \bfth_0 + \epsilon \delta \bfth$ for a perturbation magnitude $\epsilon$, where $\bfth_0 = \expectation{\bfth}_{\bfTh}$ and $\expectation{\cdot}_{\bfTh}$ denotes the expectation value over the distribution of the parameters.
The Maclaurin series of the predictions $\bfg(\bfth; \bfy)$ in $\epsilon$ is given by
\begin{equation*}
    \bfg(\bfth; \bfy) = \bfg(\bfth_0; \bfy) + J_{\bfg}(\bfth_0) \epsilon\delta\bfth + \bigo{\epsilon^2},
\end{equation*}
where $J_{\bfg}(\bfth)$ is the Jacobian matrix of the mapping $\bfg(\bfth; \bfy)$ with respect to the parameters $\bfth$, calculated in a similar manner as Eq.~\eqref{eq:jacobian}.
Then, the expectation value of the predictions $\bfg(\bfth; \bfy)$ can be expressed as
\begin{equation*}
    \expectation{\bfg}_{\bfTh} = \bfg(\bfth_0) + \bigo{\epsilon^2},
\end{equation*}
where we have set $\expectation{\delta\bfth}_{\bfTh} = 0$, given that $\bfth_0 = \expectation{\bfth}_{\bfTh}$, and set $\expectation{\bigo{\epsilon^2}}_{\bfTh} = \bigo{\epsilon^2}$ since $\bigo{\epsilon^2}$ is independent of $\bfth$.
Furthermore, using the definition of the covariance matrix,
\begin{equation*}
    \begin{aligned}
      \Cov(\bfg) &= \expectation{(\bfg - \expectation{\bfg}_{\bfTh}) (\bfg - \expectation{\bfg}_{\bfTh})^T}_{\bfTh} \\
	      &= \expectation{\left(J_{\bfg}(\bfth_0) \epsilon\delta\bfth + \bigo{\epsilon^2}\right) \left(J_{\bfg}(\bfth_0) \epsilon\delta\bfth + \bigo{\epsilon^2}\right)^T}_{\bfTh} \\
	      &= J_{\bfg}(\bfth_0) \expectation{(\epsilon\delta\bfth)(\epsilon\delta\bfth)^T}_{\bfTh} J_{\bfg}(\bfth_0)^T + \expectation{\bigo{\epsilon^3}}_{\bfTh}.
    \end{aligned}
\end{equation*}
By noticing $\expectation{(\epsilon\delta\bfth)(\epsilon\delta\bfth)}_{\bfTh} = \Cov(\bfth)$ and  $\expectation{\bigo{\epsilon^3}}_{\bfTh} = \bigo{\epsilon^3}$, the predictions uncertainty of the QoIs is thus given by the covariance matrix
\begin{equation}
    \label{eq:uncertainty_target}
    \Cov(\bfg) = J_{\bfg}(\bfth_0) \Cov(\bfth) J_{\bfg}^T(\bfth_0) + \bigo{\epsilon^3}.
\end{equation}

\section{Information-matching method}
\label{sec:info_match}

In many situations, collecting data is an expensive process, while the resulting data can often be redundant.
Optimal experimental design (OED) and active learning (AL) provide an effective strategy for data acquisition.
These methodologies help identify the most important data to collect in order to meet specific criteria and improve overall efficiency.

In this work, we introduce an information-matching approach that identifies a minimal set of data containing the necessary information to achieve the desired precision for the QoIs.
This approach leverages the FIM in Eq.~(\ref{eq:fim_least-squares}) and that representing the required information to attain the target precision of the QoIs,
\begin{equation}
    \label{eq:fim_target}
    \CalJ (\bfth) = J_{\bfg}^T(\bfth) \bfS^{-1} J_{\bfg}(\bfth),
\end{equation}
where $\bfS$ is the target covariance matrix of the QoIs.
The optimal data are selected by solving the following convex problem for the weight vector $\bfw = \begin{bmatrix} w_1 & w_2 & \dots & w_M \end{bmatrix}^T$,
\begin{equation}
    \label{eq:convex_opt}
    \begin{aligned}
	& \text{minimize} && \Vert \bfw \Vert_1 \\
	& \text{subject to} && w_m \geq 0, \\
	& && \CalI \succeq \CalJ,
    \end{aligned}
\end{equation}
where the matrix inequality implies $\CalI - \CalJ$ is positive semidefinite.
Data points with zero weights imply that infinite uncertainties are required in the label measurements, indicating that these data are insignificant.
Conversely, non-zero weights identify the most important data points and the precision with which they must be measured to ensure the target precision in the QoIs.

Numerically, we solve Eq.~\eqref{eq:convex_opt} using the CVXPY Python package \cite{diamond2016cvxpy,agrawal2018rewriting} with either the Semidefinite Programming Algorithm (SDPA) \cite{doi:10.1080/1055678031000118482,Yamashita2012,doi:10.1109/CACSD.2010.5612693,Kim2011} or Splitting Conic Solver (SCS) \cite{ocpb:16,odonoghue:21,aa2020}.
A modification of Eq.~\eqref{eq:convex_opt} with an additional binary constraint on $w_m$ can also be used, in which case the SCIP-SDP solver \cite{scip,scip_sdp} can be applied to solve the resulting mixed-integer semidefinite programming (MISDP) problem.
These solvers reformulate Eq.~\eqref{eq:convex_opt} into its equivalent dual problem and introduce Lagrange multipliers (dual values) to enforce the constraints during the optimization process.
Although minimizing the $\ell_1$-norm of the weight vectors encourages sparsity, in practice, many weights remain very small but are treated as zero by the solvers.
To unambiguously identify the weights that are effectively nonzero, we examine the dual values of the weights corresponding to the non-negativity constraint (denoted as $\tilde{w}_m$) and compare them with the solver's tolerance.
For a solver tolerance $\varepsilon$, the weight $w_m$ is considered effectively nonzero if $w_m > \varepsilon$ and $\tilde{w}_m < \varepsilon$.
The first condition ensures $w_m$ is distinguishable from zero within the solver’s precision, while the second indicates that the non-negativity constraint on $w_m$ is not binding, implying the weight is naturally positive.

\section{Theorem and proof}
\label{sec:proof}

\begin{theorem}
    \label{thm:information_matching}
    Let $\bfg(\bfth; \bfy)$ denote a mapping from the model parameters $\bfth$ to the QoIs that is analytic at $\bfth_0 = \expectation{\bfth}_{\bfTh}$, where $\expectation{\cdot}_{\bfTh}$ denotes an expectation value over the distribution of parameters.
    Consider parameters of the form $\bfth = \bfth_0 + \epsilon \delta \bfth$.
    If the constraints in Eq.~\eqref{eq:convex_opt} are satisfied, then
    \begin{equation}
	\label{eq:preds_constraint}
	\Cov(\bfg) \preceq \bfS + \bigo{\epsilon^3},
    \end{equation}
    where $\bfS$ is the target covariance of the QoIs.
\end{theorem}

The strategy to proof Theorem~\ref{thm:information_matching} is to first show that the constraints in Eq.~\eqref{eq:convex_opt} leads to
\begin{equation}
    \label{eq:preds_constraint_lin}
    J_{\bfg} \Cov(\bfth) J_{\bfg}^T \preceq \bfS.
\end{equation}
Then, we show that Eq.~\eqref{eq:preds_constraint} can be obtained by combining Eqs.~\eqref{eq:uncertainty_target} and \eqref{eq:preds_constraint_lin}.

We start by considering the eigenvalue decomposition $\CalI = \CalV \Lambda \CalV^T$ and partitioning $\CalV$ and $\Lambda$ as
\begin{equation*}
    \CalV = \begin{bmatrix} \CalV_1 & \CalV_2 \end{bmatrix}
    \quad \text{and} \quad
    \Lambda = \begin{bmatrix} \Lambda_1 & 0 \\ 0 & 0 \end{bmatrix}.
\end{equation*}
The columns of $\CalV_1$ and $\CalV_2$ span the column space and nullspace of $\CalI$, respectively.
Then, from the matrix inequality constraint in Eq.~\eqref{eq:convex_opt}, we multiply both sides with $\CalV_1^T$ on the left and $\CalV_1$ on the right,
\begin{equation}
    \label{eq:mat_ineq}
    \CalV_1^T \CalI \CalV_1 \succeq \CalV_1^T \CalJ \CalV_1.
\end{equation}
The left-hand side of this equation is invertible, while the right-hand side is not necessarily, which is a consequence of the following Lemma.

\begin{lemma}
    \label{lemma:nullspace_fims}
    Given two positive semidefinite matrices $\CalI$ and $\CalJ$.
    If $\CalI \succeq \CalJ$, then
    \begin{equation*}
	N(\CalI) \subseteq N(\CalJ),
    \end{equation*}
    where $N(\cdot)$ denotes the nullspace of the argument matrix.
\end{lemma}

\begin{proof}
    Let $v \in N(\CalI)$.
    By definition, the matrix inequality $\CalI \succeq \CalJ$ implies
    \begin{align*}
      & v^T \left( \CalI - \CalJ \right) v \geq 0 \\
      \Rightarrow & v^T \CalI v - v^T \CalJ v \geq 0 \\
      \Rightarrow & - v^T \CalJ v \geq 0.
    \end{align*}
    Since $\CalJ$ is positive semidefinite, it implies that $\CalJ v = 0$ and $v \in N(\CalJ)$.
    Next, suppose $u \in N(\CalJ)$.
    With similar steps, we arrive at $u^T \CalI u \geq 0$, which is always satisfied for a positive definite matrix $\CalI$.
\end{proof}

Continuing from Eq.~\eqref{eq:mat_ineq}, it is known that if the right-hand side is also invertible, then it follows that \cite{horn_johnson_2013,bhatia_2007}
\begin{equation}
    \label{eq:mat_ineq_inv}
    \begin{aligned}
      \CalV_1^T \CalI \CalV_1 &\succeq \CalV_1^T \CalJ \CalV_1 \\
      \Rightarrow (\CalV_1^T \CalI \CalV_1)^{-1} &\preceq (\CalV_1^T \CalJ \CalV_1)^{-1} \\
      \Rightarrow (J_{\bfg} \CalV_1) (\CalV_1^T \CalI \CalV_1)^{-1} (J_{\bfg} \CalV_1)^T &\preceq (J_{\bfg} \CalV_1) (\CalV_1^T \CalJ \CalV_1)^{-1} (J_{\bfg} \CalV_1)^T.
    \end{aligned}
\end{equation}
However, as previously stated, the right-hand side of Eq.~\eqref{eq:mat_ineq} may be singular.
However, recognizing that $\CalV_1^T \CalI \CalV$ and $\tilde{J}_{\bfg} \CalV_1$, where $\tilde{J}_{\bfg} = \bfS^{-1/2} J_{\bfg}$, share the same nullspace, we propose a more general form of the inequality in Eq.~\eqref{eq:mat_ineq_inv} as
\begin{equation}
    \label{eq:mat_ineq_pseudoinv}
    \begin{aligned}
      (\tilde{J}_{\bfg} \CalV_1) (\CalV_1^T \CalI \CalV_1)^{-1} (\tilde{J}_{\bfg} \CalV_1)^T
      &\preceq
	(\tilde{J}_{\bfg} \CalV_1) (\CalV_1^T \CalJ \CalV_1)^{\dagger} (\tilde{J}_{\bfg} \CalV_1)^T \\
      \Rightarrow
      (J_{\bfg} \CalV_1) (\CalV_1^T \CalI \CalV_1)^{-1} (J_{\bfg} \CalV_1)^T
      &\preceq
      (J_{\bfg} \CalV_1) (\CalV_1^T \CalJ \CalV_1)^{\dagger} (J_{\bfg} \CalV_1)^T,
    \end{aligned}
\end{equation}
where $(\cdot)^{\dagger}$ denotes the Moore--Penrose pseudo-inverse and we obtain the last line by multiplying both sides by $\bfS^{1/2}$ on the left and right.

The matrix on the left-hand side of Eq.~\eqref{eq:mat_ineq_pseudoinv} is the same as the left-hand side of Eq.~\eqref{eq:preds_constraint_lin}.
This can be shown by considering a parameter transformation
\begin{equation*}
    \bfth = \CalV \bfph
    = \begin{bmatrix} \CalV_1 & \CalV_2 \end{bmatrix} \begin{bmatrix} \bfph_1 \\ \bfph_2 \end{bmatrix},
\end{equation*}
which separates the components of $\bfth$ in the column space and nullspace of $\CalI$, i.e., the identifiable and unidentifiable parameters, respectively.
The covariance matrices for $\bfth$ and $\bfph$ are related by
\begin{equation*}
    \Cov(\bfth) = \CalV \Cov(\bfph) \CalV^T.
\end{equation*}
Then, we multiply both sides by $J_{\bfg}$ on the left and its transpose on the right,
\begin{equation}
    \label{eq:j_cov_j}
    \begin{aligned}
      J_{\bfg} \Cov(\bfth) J_{\bfg}^T &= J_{\bfg} \CalV \Cov(\bfph) \CalV^T J_{\bfg}^T \\
			    &= J_{\bfg} \begin{bmatrix} \CalV_1 & \CalV_2 \end{bmatrix}
			      \begin{bmatrix}
				\Cov(\bfph_1) & \Cov(\bfph_1, \bfph_2) \\
				\Cov(\bfph_2, \bfph_1) & \Cov(\bfph_2)
			      \end{bmatrix}
			      \begin{bmatrix} \CalV_1^T \\ \CalV_2^T \end{bmatrix} J_{\bfg}^T \\
			    &= J_{\bfg} \CalV_1 \Cov(\bfph_1) \CalV_1^T J_{\bfg}^T,
    \end{aligned}
\end{equation}
where we have use the fact that $J_{\bfg}$ and $\CalJ$ share the same nullspace from Eq.~\eqref{eq:fim_target} and applied Lemma~\ref{lemma:nullspace_fims} to set $J_{\bfg} \CalV_2 = 0$.
Additionally, we can relate the FIM for $\bfph$ and $\CalI$ through a similarity transformation
\begin{equation*}
    \CalI_{\bfph} = \CalV^T \CalI \CalV
    = \begin{bmatrix} \CalV_1^T \CalI \CalV_1 & 0 \\ 0 & 0 \end{bmatrix}
    = \begin{bmatrix} \Lambda_1 & 0 \\ 0 & 0 \end{bmatrix},
\end{equation*}
and it follows that the covariance of $\bfph_1$ is given by
\begin{equation}
    \label{eq:cov_phi_1}
    \Cov(\bfph_1) = (\CalV_1^T \CalI \CalV_1)^{-1}.
\end{equation}
Finally, we substitute Eq.~\eqref{eq:cov_phi_1} into Eq.~\eqref{eq:j_cov_j} to show that the left-hand sides of Eqs.~\eqref{eq:mat_ineq_pseudoinv} and Eq.~\eqref{eq:preds_constraint_lin} are the same.

For the expression on the right-hand side of Eq.~\eqref{eq:mat_ineq_pseudoinv}, we consider the singular value decomposition
\begin{equation*}
    J_{\bfg} \CalV_1 = U S V^T
\end{equation*}
and partition each matrix as
\begin{equation*}
    U = \begin{bmatrix} U_a & U_b \end{bmatrix}, \quad
    V = \begin{bmatrix} V_a & V_b \end{bmatrix}, \quad \text{and} \quad
    S = \begin{bmatrix} S_a & 0 \\ 0 & 0 \end{bmatrix}.
\end{equation*}
We use this partition and expand the right-hand side of Eq.~\eqref{eq:mat_ineq_pseudoinv},
\begin{align*}
  (J_{\bfg} \CalV_1) (\CalV_1^T \CalJ \CalV_1)^{\dagger} (J_{\bfg} \CalV_1)^T
  &= (J_{\bfg} \CalV_1) \left( (J_{\bfg} \CalV_1)^T \bfS^{-1} (J_{\bfg} \CalV_1) \right)^{\dagger} (J_{\bfg} \CalV_1)^T \\
  &= (J_{\bfg} \CalV_1) (J_{\bfg} \CalV_1)^{\dagger} \bfS \left( (J_{\bfg} \CalV_1)^T \right)^{\dagger} (J_{\bfg} \CalV_1)^T \\
  &= U_a U_a^T \bfS U_a U_a^T.
\end{align*}
Notice that the use of Moore--Penrose pseudo-inverse is justified because the nullspace of $\CalV_1^T \CalJ \CalV_1$ coincides with the nullspace of $J_{\bfg} \CalV_1$.
Additionally, since $U_a U_a^T \preceq \mathbbm{1}$, then we have
\begin{equation}
    \label{eq:mat_ineq_rhs}
    (J_{\bfg} \CalV_1) (\CalV_1^T \CalJ \CalV_1)^{\dagger} (J_{\bfg} \CalV_1)^T \preceq \bfS.
\end{equation}

Finally, by substituting Eqs.~\eqref{eq:j_cov_j} and \eqref{eq:mat_ineq_rhs} into Eq.~\eqref{eq:mat_ineq_pseudoinv}, we obtain
\begin{align*}
  (J_{\bfg} \CalV_1) (\CalV_1^T \CalI \CalV_1)^{-1} (J_{\bfg} \CalV_1)^T
  &\preceq
    (J_{\bfg} \CalV_1) (\CalV_1^T \CalJ \CalV_1)^{\dagger} (J_{\bfg} \CalV_1)^T \\
  \Rightarrow J_{\bfg} \Cov(\bfth) J_{\bfg}^T
  &\preceq (J_{\bfg} \CalV_1) (\CalV_1^T \CalJ \CalV_1)^{\dagger} (J_{\bfg} \CalV_1)^T
    \preceq \bfS,
\end{align*}
and we recover the inequality in Eq.~\eqref{eq:preds_constraint_lin}.
Then, by substituting Eq.~(\ref{eq:preds_constraint_lin}) into Eq.~(\ref{eq:uncertainty_target}),
\begin{equation*}
    \begin{aligned}
      \Cov(\bfg) &= J_{\bfg}(\bfth_0) \Cov(\bfth) J_{\bfg}^T(\bfth_0) + \bigo{\epsilon^3} \\
		 &\preceq \bfS + \bigo{\epsilon^3} \\
      \Rightarrow \Cov(\bfg) &\preceq \bfS + \bigo{\epsilon^3}.
    \end{aligned}
\end{equation*}
With this, we complete the proof of Theorem~\ref{thm:information_matching}.

\section{Model details and Other results}
\label{sec:other-results}

In this section, we present the details of the models used in the main document and additional results for applying the information-matching method in power systems networks, underwater acoustics, and interatomic potentials in materials science.

\subsection{Power systems}
\label{sec:power-systems}

Many power systems models consist of a network of buses---representing topological nodes---connected to each other by transmission lines and transformers---representing topological edges (see Fig.~\ref{fig:ieee_14bus} for an example).
Generators inject power (similarly, current) at buses, which flows through the network to loads, drawing power from the network.
In most cases, the network carries alternating current (AC) oscillating at some nominal frequency (e.g., 60 Hz in the US or 50 Hz in European countries).
In steady state operation, system quantities can be represented as complex quantities called \emph{phasors}.
Phasor can be decomposed into either real and imaginary parts, or magnitude and phase angle, which represent the lead or lag of the quantity's oscillation relative to some reference.

One of the problems in the management and operation of power grid systems is determining where to place sensors, known as Phasor Measurement Units (PMUs), to achieve complete observability of the grid.
Complete observability means that the voltage phasors at all buses can be determined \cite{yuill_optimal_2011}.
Since PMUs are expensive, the objective is to find a minimal number of PMUs that achieve this objective.
At their associated bus, PMUs are able to measure the voltage phasor and the current phasors on each adjoining branch (transmission line or transformer).
From these measurements, system equations, known as the power flow equations (based on conservation of power at buses), can be used to determine the voltage phasors at nearby buses.

The optimal PMU placement problem can be formulated as an OED problem as follows.
Bus voltage magnitudes and angles are the model's parameters $\bfth$.
These are related to the observations $\bfp_m$  made at bus $\bfx_m$ via a observation function $\bff(\bfth; \bfx_m)$,
\begin{equation}
    \bfp_m = \bff(\bfth; \bfx_m) + \epsilon_m,
\end{equation}
where $\bfp_m$ consists of components of the voltage phasor and the current phasors on adjoining branches, and $\epsilon_m$ represents measurement noise.
The objective of the optimal placement problem is to achieve full observability of the system state variables.
In this context, the QoIs are the state variables themselves, i.e., $\bfg(\bfth; \bfy) = \bfth$, and the objective is equivalent to requiring a non-singular $\CalI$.
We achieve this requirement by setting the QoI FIM $\CalJ = \lambda I$ for some small $\lambda > 0$, e.g., $\lambda = 10^{-5}$.
The positive semidefinite constraint in Eq.~\eqref{eq:convex_opt} guarantees that the eigenvalues of $\CalI$ are greater than or equal to $\lambda$, leading to a non-singular $\CalI$.

An extension to this problem involves partitioning the network into several smaller areas and determining the optimal PMU locations for identifying the state variables within each area.
In this case, we are not concerned with observing states corresponding to buses outside the area of interest; thus, we assign infinite target uncertainty to those state variables.
This is equivalent to setting the diagonal elements of $\CalJ$ corresponding to these parameters to zero.
Additionally, the candidate PMU locations are restricted to buses within the area of interest.

In this work, we consider two power network examples: the IEEE 14-bus \cite{ieee14bus} and 39-bus \cite{ieee39bus} systems.
The results for the IEEE 39-bus system are discussed in detail in the main paper.
Additional results for the IEEE 14-bus system are presented in Fig.~\ref{fig:ieee_14bus}.
The result from the information-matching approach for full system observability agrees with previous studies \cite{yuill_optimal_2011,hajian_optimal_2007,baldwin_power_1993}, as indicated by the orange highlights on buses 2, 6, and 9.
We also investigate the observability of subsets within this network, as illustrated by areas enclosed in colored (red and green) curves in Fig.~\ref{fig:ieee_14bus} \cite{transtrum_simultaneous_2018}.
The optimal PMU placements for observing each subset are shown with highlighted buses corresponding to their respective color of the area.
Notably, the optimal buses for Area B coincide with those for observing the entire network, while only one PMU with the most connections is sufficient to fully observe Area A.

\begin{figure}[!ht]
    \centering
    \includegraphics[width=0.5\textwidth]{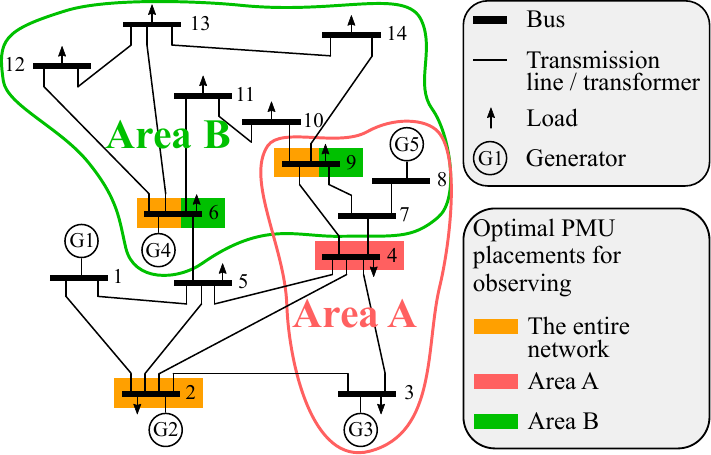}
    \caption[The IEEE 14-bus system]{
	\textbf{The IEEE 14-bus system.}
	Buses are represented by heavy black lines. Transmission lines and transformers are shown as thin lines between buses.
	Generators are depicted as circled labels G1 through G5 attached to buses, whereas loads are depicted as black arrows attached to buses.
        Orange highlighted buses indicate the optimal PMU placements for full observation of the entire network.
        Buses highlighted in red and green represent optimal PMU locations for observing areas A and B, respectively.
        Note that Buses 6 and 9 are double-highlighted with both orange and green, indicating that these buses are optimal for both full network and subset observability.
    }
    \label{fig:ieee_14bus}
\end{figure}

\subsection{Underwater acoustics}
\label{sec:underwater-acoustics}

In the field of underwater acoustics, we utilize the information-matching approach to determine the optimal locations of sensors (sound receivers, e.g., hydrophones) for identifying the locations of sound sources via passive sonar.
We simulate the sound propagation in the ocean using a range-independent normal-mode model for acousto-elastic sound propagation called ORCA \cite{westwood_normal_1996}.
For a given sound frequency $\nu$ and a set of ocean environmental parameters $\bfph$, ORCA solves the cylindrical wave equation with azimuthal symmetry, with a pressure-release boundary condition at the air-water interface at the top of the ocean.
The Green's function for this problem is given by
\begin{equation}
    \label{eq:orca_gf}
    p(r, z_s, z_r, \bfph; \nu) = \sqrt{\frac{2 \pi}{r}} e^{i \pi/4} \frac{1}{\rho_s} \sum_n{\frac{\bar{\psi_n}(z_r,\bfph;\nu) \bar{\psi_n}(z_s,\bfph;\nu) e^{i k_n(\bfph;\nu) r}}{\sqrt{k_n(\bfph;\nu)}}},
\end{equation}
where $z_r$ and $z_s$ are the receiver and source depths in meters, respectively, $r$ is the horizontal range between source and receiver in meters, $\rho_s$ is the water density at the source in kg/m$^3$, and $\nu$ is the measured sound frequency in Hz.
The depth-dependent mode functions $\bar{\psi_n}(z,\bfph;\nu)$ are vertical standing waves caused by the interference of downward and upward traveling waves at specific angles $\alpha_n$, and $k_n = \sin{\alpha_n}$ is the $n$-th modal eigenvalue.
The environmental parameters $\bfph$ may include information about water depth and sound speed, as well as sediment properties, such as density, sound speed, and attenuation coefficients of the sediment layer.
The real-valued quantity of transmission loss (TL) is then calculated from the modulus of the Green's function [Eq.~\eqref{eq:orca_gf}] as
\begin{equation}
    \label{eq:orca_tl}
    {\rm TL}(r,z_s,z_r,\bfph;\nu) = - 20 \log_{10} \left(\frac{|p(r,z_s,z_r,\bfph;\nu)|}{p_{\rm ref}}\right)
\end{equation}
with $p_{\rm ref}$ as the sound pressure in Pascals at 1~m from the sound source; i.e., TL is in units of dB re 1~m.

In this work, we aim to determine the optimal receiver locations to localize two sound sources separated vertically at depths of 8 and 16~m within a target accuracy of $\pm$ 2.5~m vertically and $\pm$ 100~m horizontally.
The candidate input data $\bfx_m$ consists of a generated rectangular grid of possible receiver locations, motivated by the common practice of using vertical and horizontal line arrays \cite{wood_optimisation_2003,dosso_optimal_1999,barlee_array_2002,dosso_array_2006}, where each input $\bfx_m$ provides the depth of receiver $m$ from the surface.
The model output $\bff(\bfth; \bfx_m)$ for each receiver povides TL measurements [Eq.~(\ref{eq:orca_tl})] for both sound sources.
This model depends on the parameters $\bfth = \{(z_s)_1, r_1, (z_s)_2, r_2, \log(\bfph)\}$, where $(z_s)_i$ and $r_i$ are the depth and range (distance between the source and receiver) of source $i$, respectively.
The FIM $\CalI_m$ for each receiver is calculated using Eq.~\eqref{eq:fim_candidate}, with additional preconditioning applied to the numerical derivative of TL with respect to the environmental parameters for improved stability, as described in \cite{mortenson_accurate_2023}.
Since our primary objective is to infer the source locations, we set the QoIs as the model parameters, i.e., $\bfg(\bfth; \bfy) = \bfth$, which also include the environmental parameters.
However, precise inference of the environmental parameters is not required; they are only estimated as needed.
Thus, the QoI FIM $\CalJ$ is set as a diagonal matrix, where the diagonal elements corresponding to the source positions [$(z_s)_i$ and $r_i$ for each source] are set to their inverse target precision, while the diagonal elements corresponding to the environmental parameters are set to zero.

Our study covers various scenarios, considering different ocean environments and sound frequencies.
In all cases, we use the same water depth and sound speed at 75~m and 1,500~m/s, respectively.
Although the seafloor composition may contain multiple layers, only a single layer (35~m thick) over a half-space is used in this work.
The half-space is characterized by a sound speed of 5,250~m/s, a density of 2.7~g/cm$^3$, and an attenuation coefficient of 0.02 dB/m-kHz.
The parameters corresponding to each seafloor material considered, including sound speed, density, and attenuation, are detailed in Table~\ref{tab:orca_sediment_parameters}.
The optimal receiver placements for each scenario are shown in Fig.~\ref{fig:orca}, with columns representing different measured sound frequencies and rows representing various sediment materials.
Notably, our findings indicate that as many as 8\% of the total receivers are sufficient to accurately localize the two sound sources within the target accuracy across different scenarios.

\begin{table}[!h]
    \centering
    \begin{tabular}{ m{15em} C{5em} C{5em} C{5em} C{5em} C{5em}}
      \hhline{======}
       & \multicolumn{5}{c}{Sediment type} \\
      Parameter name & Mud & Clay & Silt & Sand & Gravel \\
      \hline
      Top sound speed (m/s) & 1,485 & 1,500 & 1,575 & 1,650 & 1,800 \\
      Bottom sound speed (m/s) & 1,520 & 1,535 & 1,610 & 1,685 & 1,835 \\
      Bulk density (g/cm$^3$) & 1.6 & 1.5 & 1.7 & 1.9 & 2.0 \\
      Attenuation (dB/m-kHz) & 0.04 & 0.13 & 0.63 & 0.48 & 0.33 \\
      \hhline{======}
    \end{tabular}
    \caption[Seafloor sediment layer parameters]{\textbf{Seafloor sediment layer parameters.}}
    \label{tab:orca_sediment_parameters}
\end{table}

\subsection{Materials science}
\label{sec:materials-science}

In materials science, interatomic potentials are fundamental to atomistic scale simulations.
Atoms are treated as classical particles governed by Newtonian dynamics, with interatomic potentials defining the interaction energy between them.
Typically, these potentials are trained on energies and forces on atoms predicted from computationally demanding quantum-accurate theory, and then the same potentials are used in simulations on larger time- and length-scales to predict macroscopic properties of materials.
The accuracy of such predictions depends principally on the quality of the interatomic potentials used. Considering the high cost of generating training data from first-principle calculations, AL has been utilized to selectively acquire first-principles training data to improve the accuracy of the potentials while reducing computational expenses.
In this work, we use our information-matching approach to improve the efficiency of AL in developing interatomic potentials specifically tailored for accurate prediction of given target material properties.

Given an atomic configuration with $N$ atoms, the total potential energy of the configuration can be written as follows
\begin{equation}
    \label{eq:total_potential_energy}
    V = \sum_{\substack{i, j=1 \\ i < j}}^N \phi_2(\mathbf{r}_i, \mathbf{r}_j) + \sum_{\substack{i, j, k=1 \\ i < j < k}}^N \phi_3(\mathbf{r}_i, \mathbf{r}_j, \mathbf{r}_k) + \dots,
\end{equation}
where $\phi_n$ represents the $n$-body potential term and $\mathbf{r}_i$ denotes the position of atom $i$.
The force acting on atom $i$ is given by
\begin{equation}
    \label{eq:force}
    \mathbf{F}_i = -\nabla_i V,
\end{equation}
where the gradient is calculated with respect to the coordinates of atom $i$ \cite{Tadmor_Modeling_Materials,LeSar_2013}.
These equations for energy and forces are used in training the potentials and in subsequent atomistic simulations to compute the target material properties.

Here we focus on the development of a Stillinger--Weber (SW) potential for molybdenum disulfide (MoS$_2$) system (see Fig.~\ref{fig:mos2} for MoS$_2$ crystal structure).
Proposed specifically for covalent materials, the above many-body expansion is truncated to include only two-body and three-body terms expressed as:
\begin{equation}
    \label{eq:sw_potential}
    \begin{aligned}
      \phi_2^{IJ}(r_{ij}) &= A_{IJ} \left( B_{IJ} \left( \frac{\sigma_{IJ}}{r_{ij}} \right)^{p_{IJ}} - \left( \frac{\sigma_{IJ}}{r_{ij}} \right)^{q_{IJ}} \right) \exp \left( \frac{\sigma_{IJ}}{r_{ij} - r_{IJ}^{\text{cut}}} \right), \\
      \phi^{IJK}_3(r_{ij}, r_{ik}, \beta_{jik}) &= \lambda_{IJK} \left( \cos \beta_{jik} - \cos \beta_{IJK}^0 \right)^2 \exp \left( \frac{\gamma_{IJ}}{r_{ij} - r_{IJ}^{\text{cut}}} +\frac{\gamma_{IK}}{r_{ik} - r_{IK}^{\text{cut}}} \right),
    \end{aligned}
\end{equation}
where $r_{ij}$ is the distance between atoms $i$ and $j$, $\beta_{jik}$ is the angle between bonds $i-j$ and $i-k$, and the uppercase subscripts denote the atoms types \cite{SW_paper_1,SW_paper_2,wen_force-matching_2017_edited}.
Previously, this SW potential was trained to fit the energy and/or atomic forces of various atomic configurations of MoS$_2$ with their reference values (i.e., labels) obtained from first-principle calculations.

In the main document, we optimize 15 parameters of the SW potential: the two-body parameters $A_{IJ}$, $B_{IJ}$, $p_{IJ}$, and $\sigma_{IJ}$ for $Mo-Mo$, $Mo-S$, and $S-S$ interactions, the three-body parameters $\lambda_{IJK}$ for $Mo-S-Mo$ and $S-Mo-S$ interactions, and $\gamma$.
The remaining parameters are fixed at their nominal values, as specified in OpenKIM \cite{Tadmor_Elliott_Sethna_Miller_Becker_2011,elliott:tadmor:2011,OpenKIM_SW_MoS2_driver,OpenKIM_SW_MoS2}.
Furthermore, to address the differences in physical interpretations and units of the potential parameters, we apply a parameter transformation and set $\bfth$ to be the logarithms of the original parameters, thereby standardizing their scales.

\begin{figure}[!ht]
    \centering
    \begin{subfigure}[t]{0.4\linewidth}
	\centering
	\includegraphics[scale=0.2]{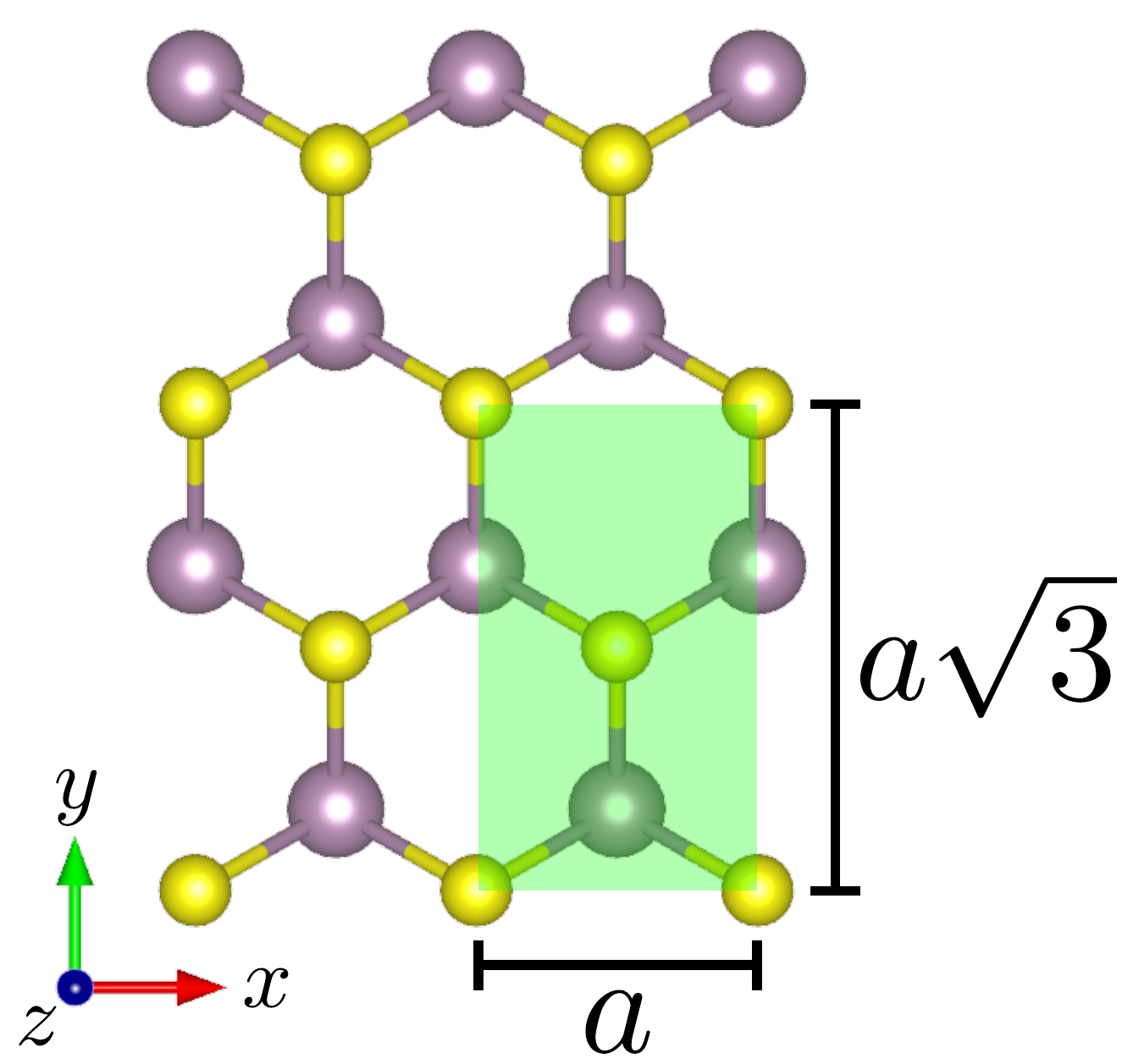}
	\caption{}
    \end{subfigure}
    \begin{subfigure}[t]{0.4\linewidth}
	\centering
	\includegraphics[scale=0.2]{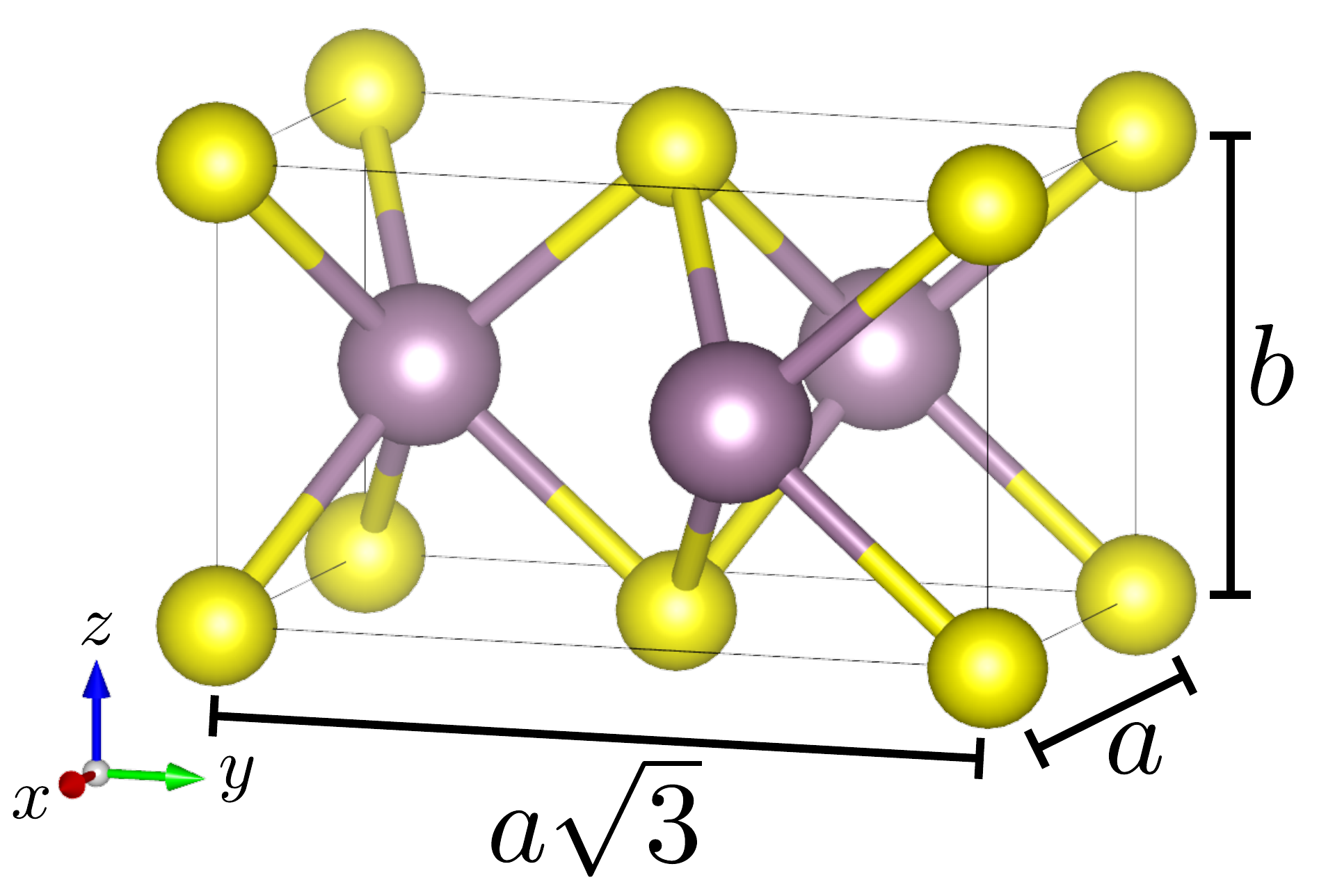}
	\caption{}
    \end{subfigure}
    \caption[Crystal structure of monolayer MoS$_2$]{
	\textbf{Crystal structure of monolayer MoS$_2$.}
	(a) Top view, with the conventional tetragonal unit cell depicted by the shaded green region.
	(b) Oblique view of the unit cell shown in (a).
	The Mo and S atoms are represented by yellow and purple spheres, respectively.
	The in-plane equilibrium lattice constant is denoted by $a$, and $b$ represents the layer thickness.
    }
    \label{fig:mos2}
\end{figure}

The candidate training dataset for the development of this potential consists of 2,000 atomic configurations obtained from snapshots of an \emph{ab initio} molecular dynamics trajectory at 750 K.
Following Wen. et al., we train this potential to fit only the forces of atoms (no energies) in these configurations, using ground truth values from density functional theory (DFT) calculations \cite{wen_force-matching_2017_edited,wen_dataset_2024}.
We then define the training model $\bff(\bfth; \bfx_m)$ to compute the force vector for each atom in a given atomic configuration $\bfx_m$ using Eq.~\eqref{eq:force}.
Finally, the Jacobian $J_{\bff}(\bfth; \bfx_m)$ to compute the FIM $\CalI_m$ is calculated by taking the derivative of this training model with respect to $\bfth$.

Our objective for this case is to precisely predict the energy as a function of lattice parameter at 0~K for a monolayer MoS$_2$.
This QoI, $\bfg(\bfth; \bfy)$,  provides a critical insight into the material's behavior and stability under varying strain conditions.
To compute the QoI, we construct an MoS$_2$ lattice in LAMMPS, an atomistic simulation library \cite{LAMMPS}, with periodic boundary conditions in the $x$ and $y$ directions (see Fig.~\ref{fig:mos2}).
Then, the sheet is compressed or stretched by varying the lattice parameter $a$ while preserving the MoS$_2$ structure.
For each value of $a$, atom positions are allowed to relax in the $z$ direction to minimize the energy, and the minimum energy is recorded.
To compute the energy increase caused by straining the lattice, we subtract from each energy value the energy computed at the equilibrium lattice constant obtained by relaxing the lattice in all directions.
We aim to precisely predict this excess strain energy within 10\% of the values predicted by the potential trained on the full dataset, and assume no correlation between the prediction points, i.e., $\bfS$ is a diagonal matrix \cite{wen_force-matching_2017_edited,OpenKIM_SW_MoS2_driver,OpenKIM_SW_MoS2}.

We apply the active learning algorithm based on information-matching, as described in the main document, to train the potential and identify the optimal training configurations.
However, we observe a discrepancy in physical units between the training data and the QoI, which can lead to convergence issues for the numerical solver applied to Eq.~\eqref{eq:convex_opt}.
To address this challenge, we modify Eq.~\eqref{eq:convex_opt} and solve the following convex problem for $\tilde{\bfw} = \begin{bmatrix} \tilde{w}_1 & \tilde{w}_2 & \dots & \tilde{w}_M \end{bmatrix}^T$,
\begin{equation}
    \label{eq:convex_opt_edited}
    \begin{aligned}
	& \text{minimize} && \Vert \tilde{\bfw} \Vert_1 \\
	& \text{subject to} && \tilde{w}_m \geq 0, \\
	& && \tilde{\CalI} = \sum_m \tilde{w}_m \tilde{\CalI}_m \succeq \tilde{\CalJ},
    \end{aligned}
\end{equation}
where $\tilde{w}_m = w_m a_m/b$, $\tilde{\CalI}_m = \CalI_m/a_m$, $\tilde{\CalJ} = \CalJ/b$, $a_m = \Vert \CalI_m \Vert_F$, $b = \Vert \CalJ \Vert_F$, and $\Vert \cdot \Vert_F$ denotes the Frobenius norm.
The Frobenius norm for a matrix $A$ is calculated by
\begin{equation*}
    \Vert A \Vert_F = \left(\sum_i \sum_j A_{ij}^2 \right)^{1/2},
\end{equation*}
where $A_{ij}$ is the element of matrix $A$ in the $i$-th row and $j$-th column.
Notably, solving Eq.~\eqref{eq:convex_opt_edited} is equivalent to solving Eq.~\eqref{eq:convex_opt}.
Additionally, the loss function in Eq. \eqref{eq:least-squares_loss} uses the weights $w_m$ rather than the transformed weights $\tilde{w}_m$.
The optimal results for this case are presented and discussed in the main document.

The outcome of the active learning loop in Algorithm~1 can depend on the initial choice of model parameters used to evaluate the FIM.
Different initial parameters can lead to different selected configurations, corresponding weights, and ultimately different optimal parameter estimates.
As a result, the final model predictions and the associated uncertainties may vary across runs.
However, regardless of these variations, the method is designed to ensure that the predicted QoI uncertainties satisfy the prescribed target precision, provided that the problem is feasible.
Thus, this method is robust to the choice of initial conditions.

To illustrate this behavior, we consider a numerical example involving several SW potentials for MoS$_2$, each constructed with the same settings as before but initialized from different randomly perturbed parameter values.
Each such initialization defines a separate run of the active learning loop.
In this example, instead of using DFT forces as ground truth, we use forces generated by a reference SW potential trained on the full dataset \cite{wen_force-matching_2017_edited}, and add independent Gaussian noise with a standard deviation of 10~meV/\AA~ to each force component.
This substitution allows us to control the variability in the force data and focus specifically on the effects of the initial parameter choice.
The target uncertainty for the QoI is again set to 10\% of the corresponding QoI values predicted by this reference potential, with a minimum threshold of 10~meV.

Across these runs, only 1--2\% of the candidate configurations are sufficient to constrain the relevant parameters and achieve the target precision.
Each run selects a different set of optimal configurations and corresponding weights; thus, they constrain the model parameters differently.
As a consequence, the QoI predictions, errors, and uncertainties also vary with the initial parameters, as shown in Fig.~\ref{fig:swmos2_initial_parameters}a.
Despite these variations, the propagated uncertainty from the selected training data is consistently lower than the target uncertainty.
This behavior is a consequence of the matrix inequality constraint imposed in Eq.~(\ref{eq:convex_opt}).
When this constraint is satisfied, the eigenvalues of the FIM for the optimal training data are larger than those of the QoI FIM, as illustrated in Fig.~\ref{fig:swmos2_initial_parameters}b.
This indicates that the model parameters are adequately constrained by the optimal training data than what is required.
Thus, while there is some dependence on the initial parameters, the final result is robust, and the main objective of uncertainty constraint is consistently satisfied.

\begin{figure}[!ht]
    \centering
    \includegraphics[width=0.9\textwidth]{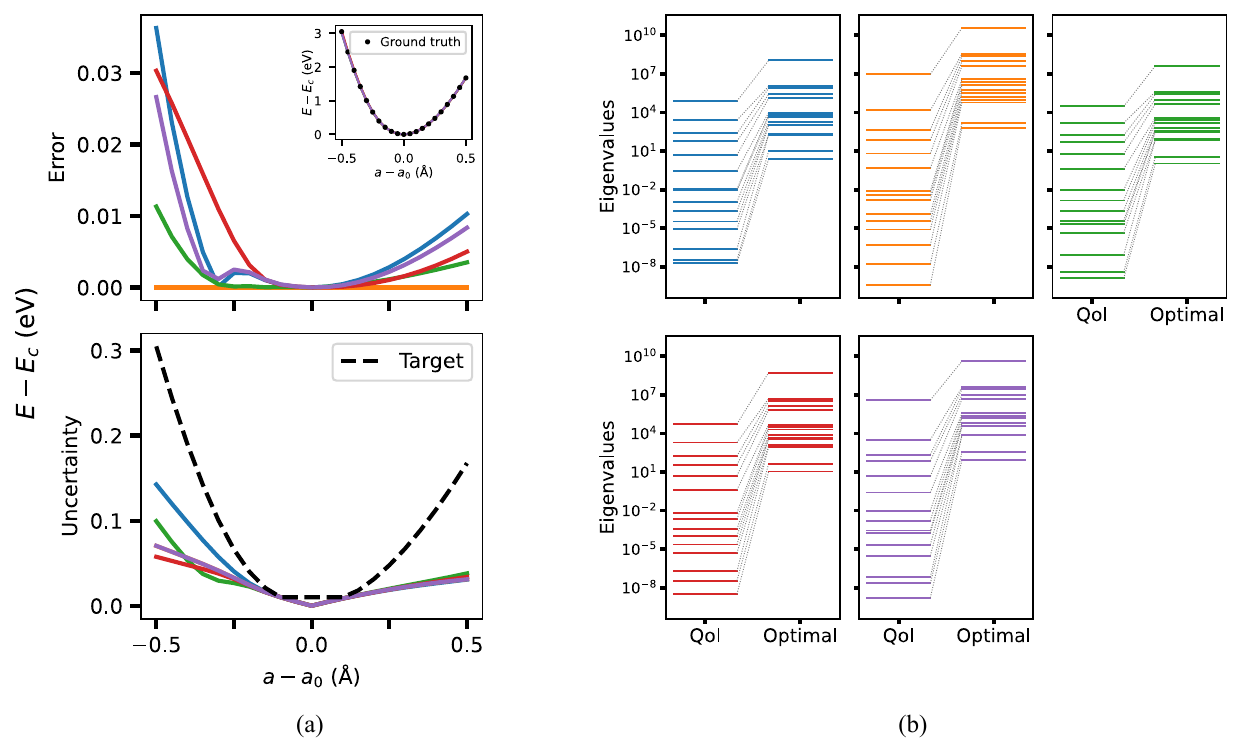}
    \caption[Effect of initial parameters on the active learning outcome for the SW MoS$_2$ potential]{
	\textbf{Effect of initial parameters on the active learning outcome for the SW MoS$_2$ potential.}

	(a) QoI prediction errors (top) and prediction uncertainties (bottom) for models initialized with different randomly perturbed parameters (colors).
	The inset in the top panel shows the corresponding energy predictions, which closely overlap across runs.
	The black dashed curve in the bottom panel represents the target uncertainty for the QoI.
	While specific predictions and uncertainties vary across initializations, the uncertainty constraint is consistently satisfied---i.e., all prediction uncertainties remain below the target.
	This consistency is a consequence of the matrix inequality constraint in Eq.~(\ref{eq:convex_opt}).

	(b) Comparison between the eigenvalues of the QoI FIM and those of the FIM computed from the optimal training data for different runs.
	When the matrix inequality constraint is satisfied, the eigenvalues of the optimal training-data FIM exceed those of the QoI FIM component-wise, indicating that the selected training data provide sufficient information to achieve the prescribed QoI precision.
    }
    \label{fig:swmos2_initial_parameters}
\end{figure}

We additionally demonstrate the performance of our information-matching algorithm in fitting a single-element SW potential for silicon (Si).
We use only five parameters of the SW potential: $A$, $B$, $\sigma$, $\lambda$, and $\gamma$, while keeping all the remaining parameters fixed at their nominal values provided in OpenKIM \cite{SW_paper_1,SW_paper_2,SW_OpenKIM_MD,SW_OpenKIM_MO}.
Similar to the MoS$_2$ case, we use logarithms of these potential parameters as the set $\bfth$ to account for their differing physical units.
Our candidate training data consists of 400 atomic configurations; 100 configurations correspond to perfect diamond cubic unit cells with varying lattice parameters, in which forces on atoms are zero and only the energy per atom [Eq.~\eqref{eq:total_potential_energy}] is used for potential fitting (an illustration of a diamond unit cell is given in Fig.~\ref{fig:si}), and 300 additional configurations were created by perturbing atoms from their perfect positions in the unit cells \cite{Wen_Afshar_Elliott_Tadmor_2022}.  
For these 300 configurations, only forces on atoms [Eq.~\eqref{eq:force}] are utilized in fitting.
For this demonstration, all labels are generated using the EDIP potential \cite{EDIP_paper,OpenKIM_EDIP_MD,OpenKIM_EDIP_MO} rather than from first-principle calculations.

\begin{figure}[!ht]
    \centering
    \includegraphics[scale=0.2]{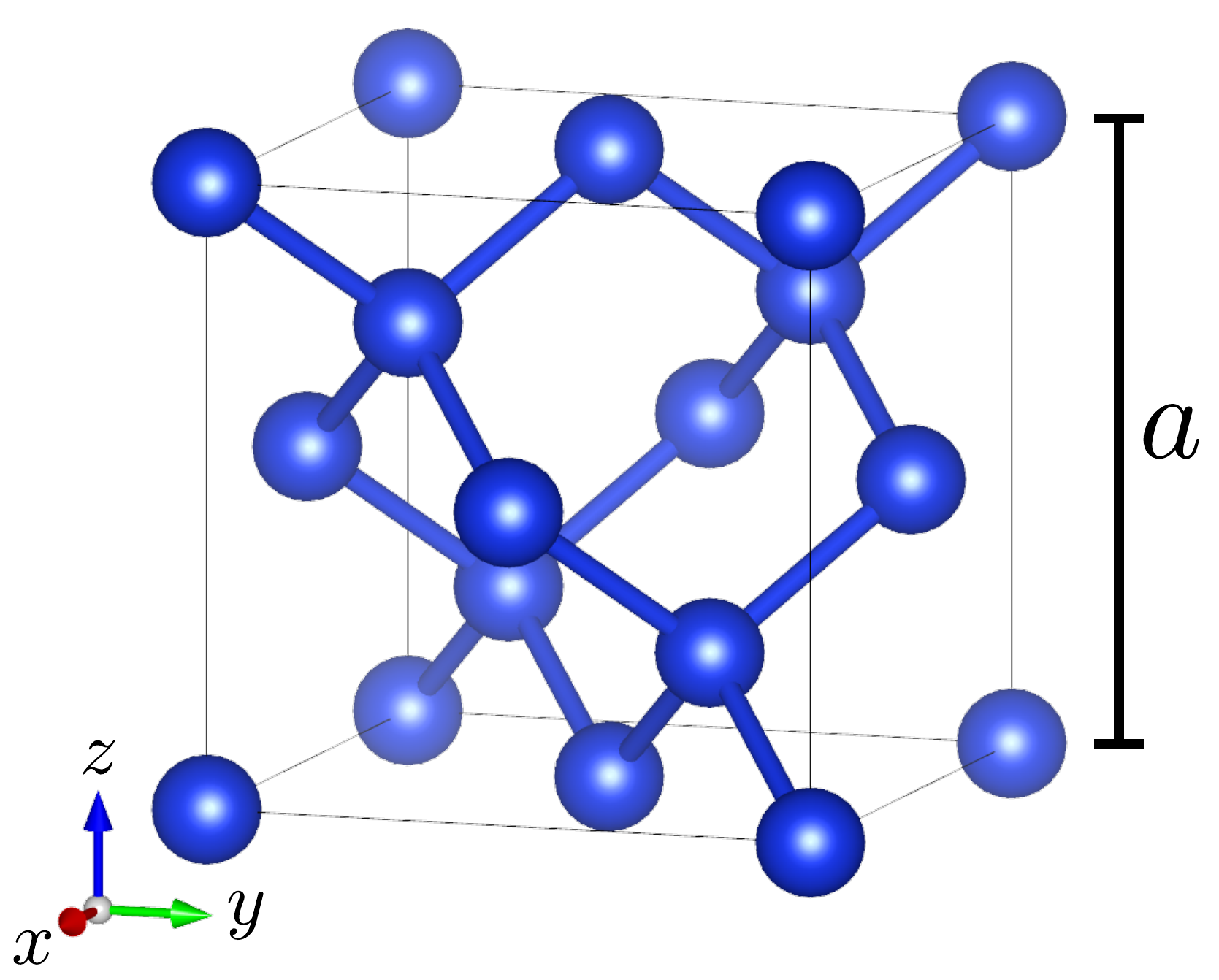}
    \caption[Diamond cubic crystal structure of Si]{
	\textbf{Diamond cubic crystal structure of Si.}
	Silicon atoms are represented as blue spheres and the lattice parameter $a$ indicates the dimension of the cubic lattice.
    }
    \label{fig:si}
\end{figure}

We develop three optimal SW potentials, each designed to precisely predict different QoIs for the diamond cubic Si through separate active learning (AL) calculations.
These QoIs include: (1) the equilibrium lattice constant and the elastic constants, (2) energy as a function of lattice parameter, and (3) phonon dispersion curves.
The equilibrium lattice constant ($a_0$) is determined by minimizing the energy with respect to the lattice parameter $a$, and the corresponding minimum energy per atom yields the cohesive energy ($E_c$).
The elastic constants are calculated from the Hessian matrix (second derivative matrix) of the energy density with respect to lattice strain.
For a cubic crystal like diamond, there are three independent elastic constants $c_{11}$, $c_{12}$, and $c_{44}$, which represent the material's stiffness under uniaxial strain, the coupling between perpendicular stresses, and the resistance to shear deformation, respectively \cite{Tadmor_Modeling_Materials,OpenKIM_elastic_constants_TD,OpenKIM_elastic_constants_TE}.
The energy as a function of lattice parameter is obtained by uniformly varying $a$ in all directions and calculating the corresponding energy per atom.
Finally, the phonon dispersion is calculated using the ASE Python package \cite{ase-paper}, which involves computing the force constants that describe the atomic force response to small displacements in the crystal.

For each QoI, the predictions are assumed to be independent, and the target precision is set to be 10\% of the values predicted by the potential developed by Stillinger and Weber \cite{SW_paper_1,SW_paper_2,SW_OpenKIM_MD,SW_OpenKIM_MO}.
This implies that, in each case, $\bfS$ is a diagonal matrix with the diagonal elements given by the target variance of the corresponding predictions.
Additionally, as in the MoS$_2$ case, we replace Eq.~\eqref{eq:convex_opt} with Eq.~\eqref{eq:convex_opt_edited} within the active learning algorithm to address the unit discrepancies between training data and QoIs.
Comparisons between the target precisions and the uncertainties obtained from the optimal configurations are presented in Table~\ref{tab:swsi_lattice_elastic_constants} for the lattice and elastic constants, and Fig.~\ref{fig:swsi_energy_latconst_phonon_dispersion} for the energy vs.~lattice parameters and phonon dispersion curve.
In all three cases, the optimal training sets consist of at most five atomic configurations with varying lattice parameters, among which one corresponds to a perfect lattice configuration with only energy data, while the others include forces on atoms data.
Together, these configurations provide sufficient information to constrain the model parameters and achieve the predefined target precision.
Furthermore, our information-matching procedure can identify which specific quantities---energy or force components---should be computed in each down selected candidate configuration.

\begin{table}[!h]
    \centering
    \begin{tabular}{m{10em} C{5em} C{5em} C{5em} C{5em} C{5em}}
      \hhline{======}
      & $a_0~(\AA)$ & $E_c$ (eV) & $c_{11}$ (GPa) & $c_{12}$ (GPa) & $c_{44}$ (GPa) \\ 
      \hline
      Optimal predictions & 5.4307 & 4.3363 & 151.4339 & 76.4375 & 56.4477 \\
      Optimal uncertainty & 0.0477 & 0.2321 & 7.9730 & 6.9252  & 5.0839 \\
      Target uncertainty & 0.54310 & 0.43364 & 15.14158 & 7.64180 & 5.64458 \\
      \hhline{======}
    \end{tabular}
    \caption[Lattice and elastic property predictions for silicon in diamond structure]{
	\textbf{Lattice and elastic property predictions for silicon in diamond structure}
	The columns show the target properties: the equilibrium lattice constant ($a_0$), cohesive energy ($E_c$), and elastic constants $c_{11}$, $c_{12}$, and $c_{44}$.
	The first and second rows present the predictions and uncertainties of these quantities calculated using the optimal SW potential trained on configurations identified by the information-matching approach.
	The optimal training configurations comprise one perfect lattice configuration and three perturbed lattice configurations, each with different lattice parameters.
	These training sets effectively constrain potential parameters to achieve the predefined target precision, as indicated by lower optimal uncertainty values compared with the target error bars presented on the third row.
    }
    \label{tab:swsi_lattice_elastic_constants}
\end{table}

\begin{figure}[!h]
    \centering
    \begin{subfigure}[t]{0.45\textwidth}
	\centering
	\includegraphics[width=\textwidth]{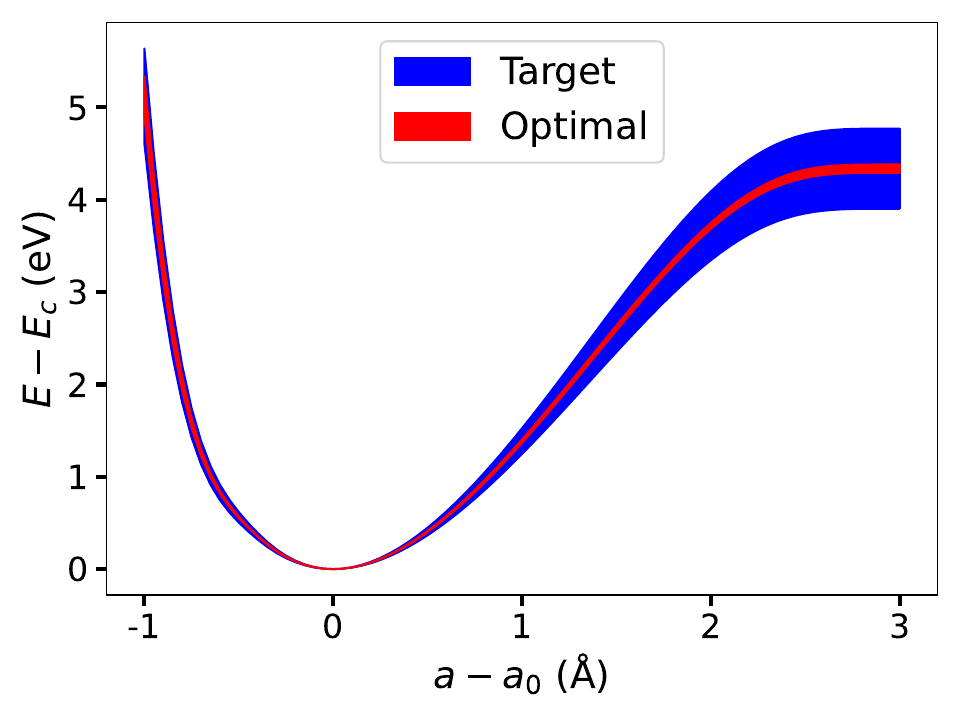}
	\caption{}
	\label{fig:swsi_energy_latconst}
    \end{subfigure}
    \qquad
    \begin{subfigure}[t]{0.45\textwidth}
	\centering
	\includegraphics[width=\textwidth]{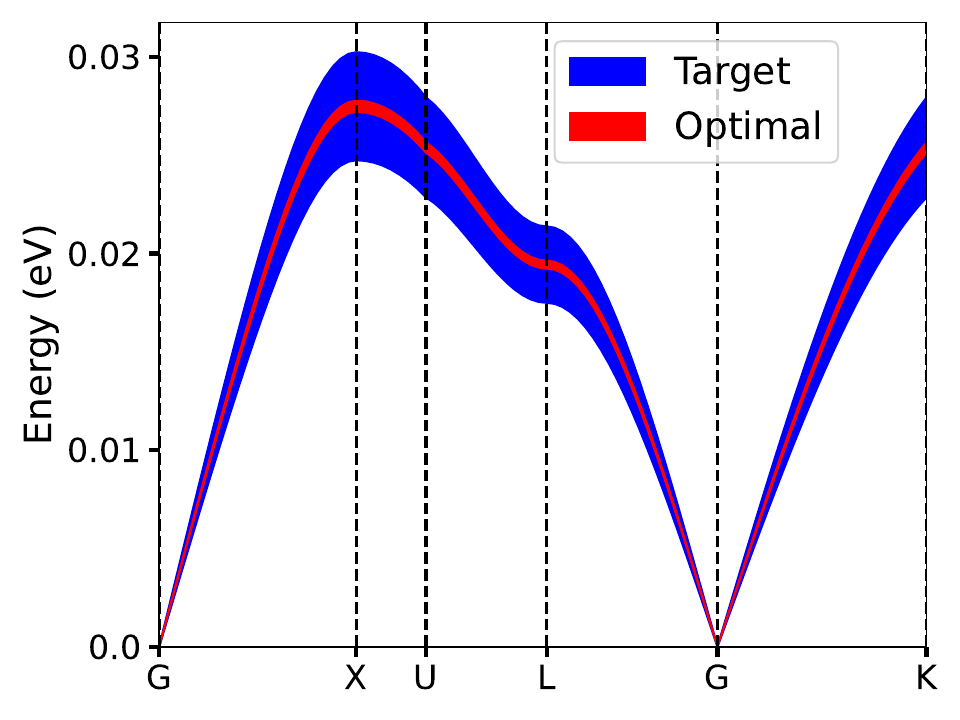}
	\caption{}
	\label{fig:swsi_phonon_dispersion}
    \end{subfigure}
    \caption[Uncertainties of the energy $E$ as a function of lattice parameter $a$ and phonon dispersion curve for silicon in diamond structure]{
	\textbf{Uncertainties of (a) the energy $E$ as a function of lattice parameter $a$ and (b) phonon dispersion curve for silicon in diamond structure.}
	The blue envelopes represent the target precision for each QoI, while the red envelop indicate the uncertainties induced by the optimal training configurations.
	For each case, the optimal configurations consist of one perfect lattice configuration and four perturbed lattice configurations, each with different lattice parameters.
	Note that in both scenarios, these five configurations sufficiently constrain the potential parameters, resulting in propagated uncertainties that are smaller than the target uncertainty.
    }
    \label{fig:swsi_energy_latconst_phonon_dispersion}
\end{figure}

\begin{figure}[!ht]
    \centering
    \includegraphics[width=0.6\textheight]{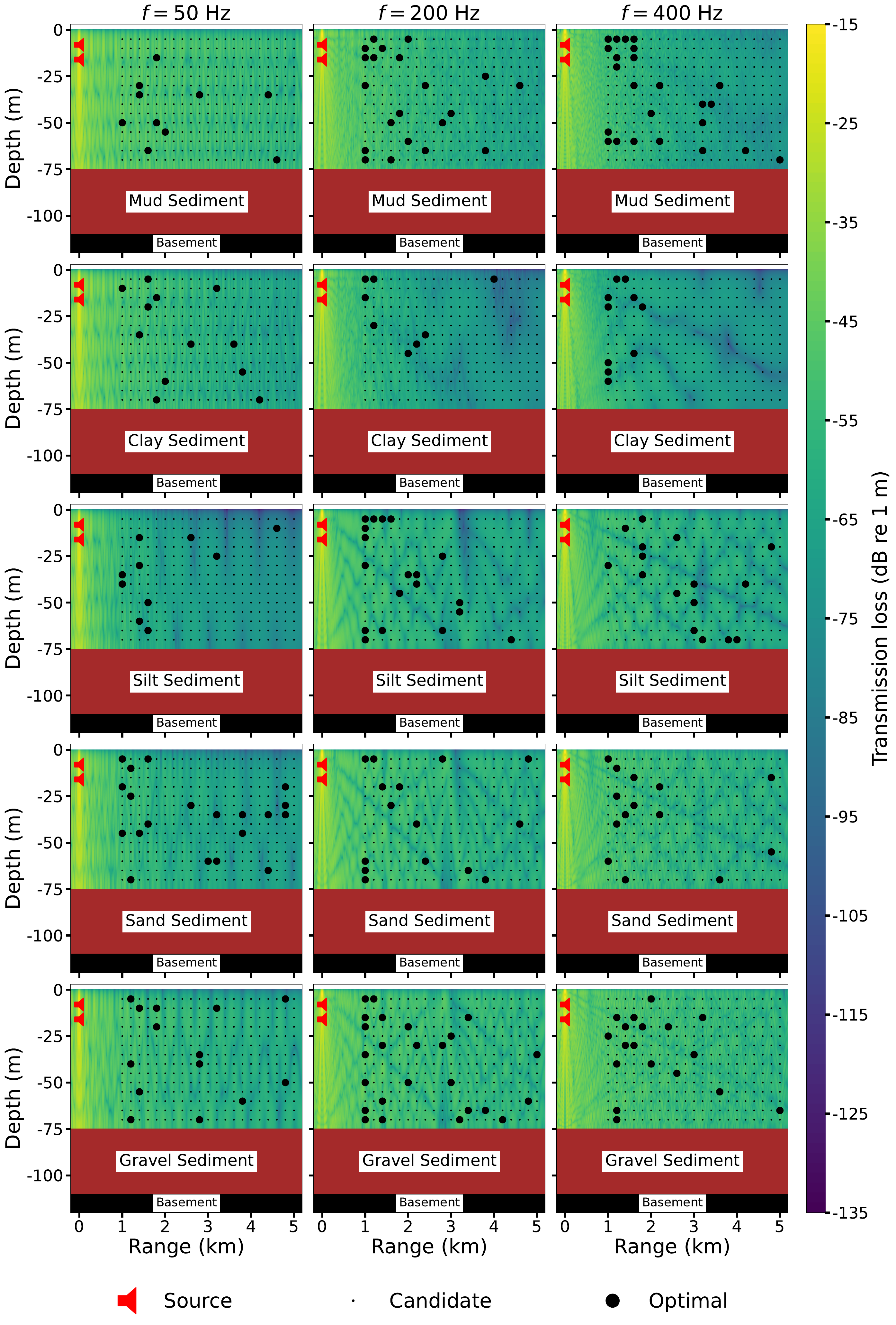}
    \caption[Source localization in a shallow ocean]{
	\textbf{Source localization in a shallow ocean.}
	Optimal receiver placements for localizing two sound sources at the top left (red speakers), considering various sediment materials and sound frequencies.
	The columns denote different source frequencies, while the rows represent sediment types.
	The objective for all cases considered is to locate each sound source within a target accuracy of $\pm$ 2.5~m vertically and $\pm$ 100~m horizontally.
	Larger dots represent the optimal receiver locations determined using the information-matching method, among the candidates denoted by smaller dots.
	Additionally, we have included the transmission loss pattern from the top source for each case to show some correlation between the optimal receiver locations and the transmission loss pattern.
	Across different scenarios, the information-matching approach indicates that we only need at most 8\% of the total receivers to locate the sources within the target accuracy.
    }
    \label{fig:orca}
\end{figure}

\bibliographystyle{unsrt}
\bibliography{refs.bib,refs_zotero.bib}